\documentclass{article} 
\usepackage{iclr2027_conference,times}

\usepackage{amsmath,amsfonts,bm}

\def\eqref#1{equation~\ref{#1}}

\def\1{\bm{1}}

\DeclareMathAlphabet{\mathsfit}{\encodingdefault}{\sfdefault}{m}{sl}
\SetMathAlphabet{\mathsfit}{bold}{\encodingdefault}{\sfdefault}{bx}{n}

\usepackage{hyperref}
\usepackage{url}
\usepackage{graphicx}
\usepackage{booktabs}
\usepackage{subcaption}
\usepackage{array}
\usepackage{amssymb}
\usepackage{amsthm,amsfonts}
\usepackage{dsfont}
\usepackage[most]{tcolorbox}
\usepackage{algorithm,algpseudocode}
\usepackage{multirow} 
\newtcolorbox{promptbox}[1]{
  breakable,
  colback=gray!3,
  colframe=black!65,
  title={#1},
  fonttitle=\bfseries,
  boxrule=0.6pt,
  arc=1mm,
  left=1.2mm,
  right=1.2mm,
  top=1.2mm,
  bottom=1.2mm
}

\newtheorem{informaltheorem}{Theorem (Informal)}
\newtheorem{proposition}{Proposition}[section]
\newtheorem{theorem}{Theorem}[section]
\newtheorem{lemma}{Lemma}[section]
\newtheorem{assumption}{Assumption}[section]
\newtheorem{corollary}{Corollary}[section]
\theoremstyle{definition}
\newtheorem{definition}{Definition}[section]
\newtheorem{remark}{Remark}[section]

\title{From Preference to Reciprocity: Decentralized Matching with Empirically Grounded LLM-agent Based Modeling}

\author{
\begin{minipage}{\textwidth}
\centering
\vspace*{15pt}
\normalsize
Wangxuan Fan\thanks{Equal contribution.}\quad
Xiaoyu Nie\footnotemark[1]\quad
Zhoutian Shi\quad
Xiangcheng Meng\\[4pt]
Shipei Zeng\quad
Pin Gao\quad
Yan Hu\thanks{Corresponding author.}\quad
Zhongxiang Dai\footnotemark[2]\\[8pt]
{\small The Chinese University of Hong Kong, Shenzhen}\\[5pt]
{\footnotesize
\texttt{\{wangxuanfan,xiaoyunie,zhoutianshi,xiangchengmeng\}@link.cuhk.edu.cn}\\[2pt]
\texttt{\{gaopin,huyan,daizhongxiang\}@cuhk.edu.cn}
\quad
\texttt{shipei.zeng@sribd.cn}
}
\end{minipage}
}
\iclrfinalcopy 

\usepackage{etoolbox}

\makeatletter
\patchcmd{\@maketitle}
  {\begin{tabular}[t]{l}\bf\rule{\z@}{24pt}\@author\end{tabular}}
  {\begin{center}\@author\end{center}}
  {}
  {\PackageError{author-layout}{Author layout patch failed}
    {Check the definition of @maketitle in the style file.}}
\makeatother

\begin{document}

\maketitle
\fancyhead{}
\lhead{Preprint. Under review.}
\renewcommand{\headrulewidth}{0pt}

\begin{abstract}
Bipartite matching is a fundamental problem in game theory and market design. Classical approaches such as Gale--Shapley assume complete preferences and centralized computation, whereas many real-world matching processes are decentralized, asynchronous, and shaped by sequential interaction under limited information. We propose a dynamic bipartite matching framework that combines large language model (LLM) agents with contextual bandits. In a simulated Chinese marriage market, economically grounded LLM agents evaluate locally encountered candidates, while agent-specific Logistic-UCB models learn reciprocal acceptance from realized proposal outcomes. The mechanism therefore separates two decisions---\emph{whom do I like?} and \emph{who is likely to like me back?}---without requiring ex ante market-wide preference rankings. We first validate LLM-induced mate preferences against the empirical conditional-logit reference across multiple LLM backbones. In the $50\times50$ matching experiment, Bandit-UCB achieves the highest mean mutual welfare (56.01 versus 54.87 for Gale--Shapley), a smaller gender rank gap than the classical baselines, and the fewest blocking pairs among the LLM-ABM policies. Learned acceptance models show economically interpretable gender-differentiated associations, while counterfactual setups reveal no systematic unilateral advantage from prior search knowledge. Overall, these results support the advantages of decentralized matching with LLM-based behavioral modeling and online learning under incomplete information for economic simulation and computational social science research.
\end{abstract}

\section{Introduction}
\label{sec:introduction}

Two-sided matching is a fundamental problem in game theory and market design, with applications ranging from labor markets and school choice to marriage markets \citep{gale1962college,kamecke1992two,roth1992two}. 
The classical Gale--Shapley deferred-acceptance algorithm \citep{gale1962college} guarantees a stable matching by taking agents' ordinal preference rankings as given. 
While this formulation has been highly successful as a normative allocation mechanism, it abstracts away from several features of naturally occurring matching processes \citep{eyupoglu2021stable, axtell2008high,zhang2024decentralized}: agents typically observe only a limited set of alternatives, interactions occur sequentially and asynchronously, and preferences or feasible opportunities may be learned only through experience. 
Moreover, stability does not necessarily coincide with aggregate welfare.
Both analytical and computational studies have shown that restricting attention to stable outcomes can impose a non-negligible welfare cost \citep{axtell2008high,boudreau2013preferences,chen2021matchings,ortega2024unimprovable}. 
These observations motivate a complementary view of matching: rather than asking only how to compute an allocation given a complete preference profile, we study how preference information, beliefs, and matches can jointly evolve through decentralized interaction.

Agent-based modeling (ABM) provides a natural bottom-up framework for such a perspective. 
By specifying heterogeneous agents, local information, and interaction rules, ABMs generate aggregate outcomes from repeated individual decisions \citep{epstein2012generative,axtell2025abm}. 
This is particularly relevant to marriage markets, where individual partner choices depend on the current pool of available alternatives, while every new relationship changes that pool for subsequent decisions \citep{kalmijn1998intermarriage, chiappori2006divorce,van2016introduction,grow2019design}. \citet{axtell2008high}, for example, replace centralized deferred acceptance with decentralized search, proposal, and rematching, achieving higher average welfare without preserving guaranteed stability, which is not necessarily a strict desideratum in marriage markets. 
Yet decentralizing the \emph{matching process} does not by itself solve the problem of \emph{preference formation}: candidate rankings still need to be specified before the matching dynamics begin. 
This is a central difficulty in agent-based marriage-market modeling. 
As emphasized by \citet{grow2019design}, theories of mate choice are often \textit{competing}, \textit{multidimensional}, and \textit{verbal} rather than directly mathematical, forcing modelers to translate behavioral assumptions into hand-designed utility functions and parameterizations.

Large language models (LLMs) provide a novel way to operationalize semantically rich behavioral assumptions.
Existing studies show that LLM agent-based modeling (LLM-ABM) can condition decisions on heterogeneous personas and contextual information and can be embedded in agents that interact over extended trajectories \citep{argyle2023out,aher2023using,park2023generative}.
However, we do not treat an LLM as an unconstrained oracle of human preferences: expressive behavior is not necessarily economically valid behavior, motivating careful specification and validation when LLMs are used as simulated economic actors \citep{horton2023large,ludwig2024large}.
Nor do we provide the agents with market-wide preference rankings in advance, or let them infer the results in one-shot.
Instead, LLM valuation is triggered only when agents encounter one another.
Yet knowing how much an agent values a candidate does not reveal whether that candidate will reciprocate.
Because the proposer observes acceptance only after making a proposal, reciprocity becomes a sequential learning problem under partial feedback.
This naturally motivates an agent-specific contextual bandit that learns, from realized interactions, which desirable candidates are likely to accept.
Unlike matching-bandit approaches that learn latent rewards or unknown preference rankings \citep{liu2021bandit,dai2021learning,cen2022regret}, our framework separates \emph{behavioral evaluation} from \emph{interaction learning}: the LLM determines how much an agent values an encountered candidate, while the bandit learns how likely that candidate is to accept.


Based on this separation, we develop a decentralized and asynchronous bipartite matching framework that combines behaviorally grounded LLM agents with agent-specific contextual bandits.
In the marriage-market setting, agents encounter only locally available alternatives, evaluate them through LLM-based preferences, and learn reciprocal acceptance from realized outcomes, without requiring complete rankings in advance.
Experiments on a simulated Chinese marriage market validate the behavioral grounding of the agents and show favorable welfare--stability outcomes relative to classical and decentralized baselines, while learned acceptance models and counterfactual experiments reveal structured reciprocal-choice patterns and the limits of unilateral information advantages.

Our contributions are threefold:
\begin{itemize}

\item \textbf{Empirically grounded LLM-agent based modeling.}
We develop a general approach for constructing and validating LLM agents from domain-specific empirical evidence, and instantiate it in a marriage market using socioeconomic characteristics and gender-specific mate-preference evidence. We validate the induced preferences against an empirical reference across multiple LLM backbones and market scales.

\item \textbf{Decoupled decentralized matching mechanism.}
We formulate matching as an asynchronous local-interaction process that separates LLM-based partner valuation from agent-specific reciprocity learning, without requiring agents to access complete preference rankings. For a theoretically calibrated variant, we establish local proposal-regret bounds under realizable and dynamically misspecified acceptance models.

\item \textbf{Economic evaluation and interpretable results.}
We show that the proposed framework preserves the core two-sided matching problem while relaxing strong information assumptions, and achieves better overall welfare--stability outcomes. The learned reciprocal-acceptance models yield economically interpretable patterns, while warm-start counterfactuals reveal how prior market knowledge translates into individual and market-level matching outcomes.

\end{itemize}
Our code is publicly available in our
\href{https://github.com/YuanJrShiuan/LLM_ABM_Bipartite_Matching_for_Marriage}
{GitHub repository}.
Related work is reviewed in Appendix~\ref{app:related_work}.

\section{Empirically Grounded LLM-Agent Modeling}
\label{sec:llm-agent-modeling}

\subsection{Synthetic Population with Empirical Heterogeneity}

Our modeling principle is to construct LLM agents from domain-specific empirical evidence rather than unconstrained persona generation. This requires identifying behaviorally relevant observable characteristics and preserving important population heterogeneity and dependence structures. The resulting structured profiles provide both the population representation and the empirical basis for downstream LLM-based behavioral evaluation.

We instantiate this approach in a synthetic Chinese marriage market. Each agent is endowed with attributes emphasized in empirical studies of mate choice \citep{zhou2023gender}, including \textit{age, income, education, family background, housing status, and physical appearance}. These dimensions capture the \textit{gender heterogeneity} and \textit{social-class differences} documented in the empirical reference.

Following the dependency-aware construction of \citet{li2026matraix}, we sample demographic variables to approximate observed population heterogeneity while preserving key relationships, such as income--education dependence, rather than treating attributes independently. Demographic and economic variables are based on the 2015 National 1\% Population Sample Survey of China \citep{nbs2015mini}, consistent with the period studied by \citet{zhou2023gender}, while physical appearance is sampled from an approximately normal distribution based on empirical evidence \citep{hamermesh1993beauty}. Full sampling details and realized population composition are reported in Appendix~\ref{app:persona_sampling}.

\subsection{LLM Personas and Mate Preference}
Each structured profile is rendered as a concise public persona containing only observable attributes available to potential partners, ensuring that LLM evaluations use the same information encoded in the demographic profile. The rendering template is provided in Appendix~\ref{app:persona_prompt}.

To model mate preferences, we construct gender-specific prompts grounded in the regression evidence of \citet{zhou2023gender} and use chain-of-thought prompting \citep{wei2022chain} to structure evaluation. The full private preference prompts are given in Appendix~\ref{app:preference_prompts}.

\begin{figure}[htbp]
    \centering
    \includegraphics[width=0.95\linewidth]{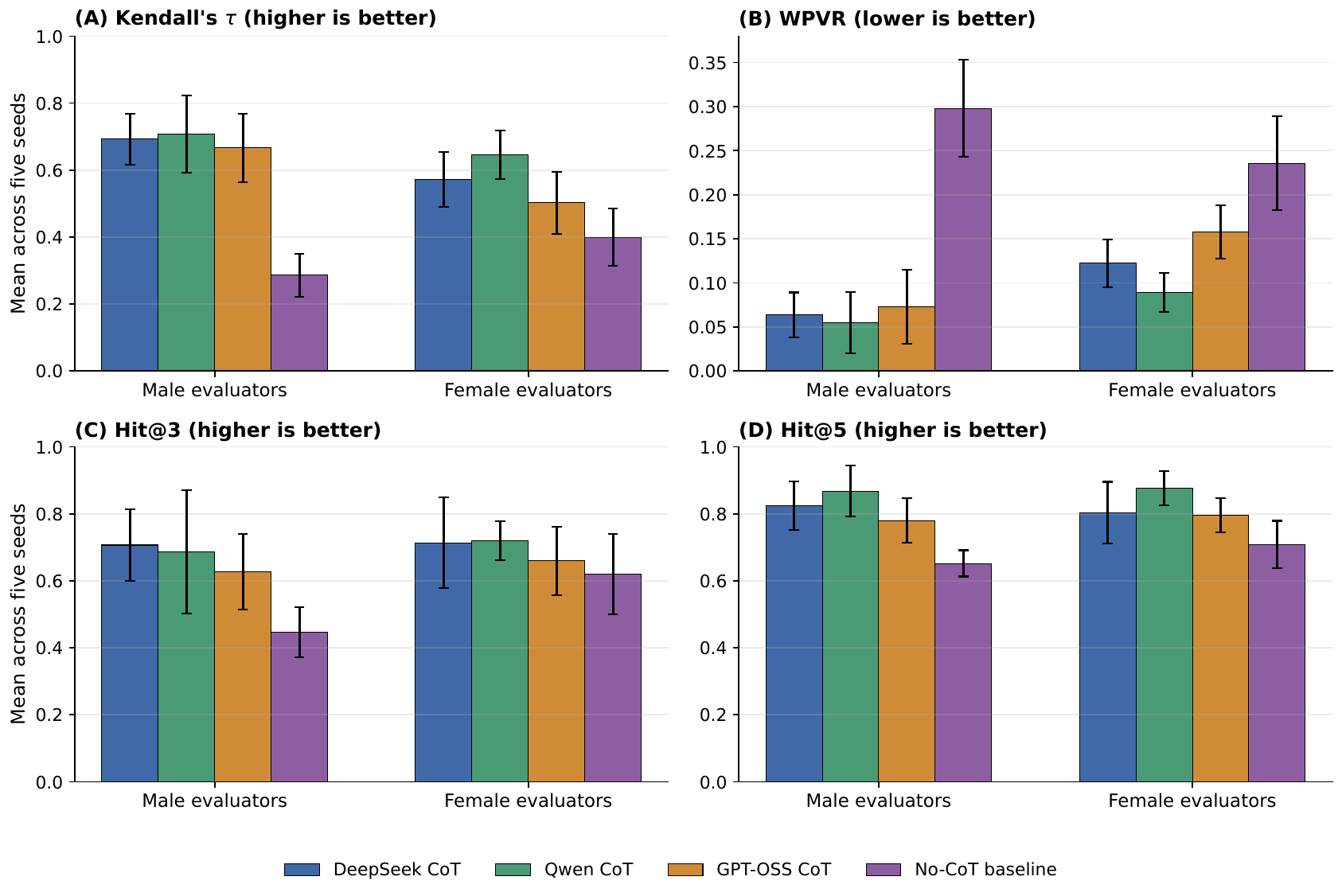}
    \caption{Empirical-reference alignment in the 10$\times$10 validation task. Bars show means across five seeds and error bars show standard deviations.}
    \label{fig:oracle_alignment_10x10}
\end{figure}

\subsection{Validation of Modeling}
\label{sec:validation_modeling}

As discussed in Appendix~\ref{app:related_work}, validating persona-conditioned decisions is important for establishing the economic interpretability of LLM-based social simulations.
We therefore compare the mate-preference rankings induced by our LLM agents with an empirical reference model derived from \citet{zhou2023gender}.
The reference model maps each candidate profile to a conditional-logit utility under the corresponding gender-specific specification, while the LLM assigns an overall desirability score from the evaluator's public persona and private mate-preference prompt. We evaluate ranking agreement using Kendall's $\tau$, weighted pairwise violation rate (WPVR), and top-$K$ overlap, with definitions provided in Appendix~\ref{app:validation_metrics}.

Figure~\ref{fig:oracle_alignment_10x10} reports the validation results for a $10\times10$ market across five random seeds.
Across all three LLM backbones, preference-structured prompts achieve stronger alignment with the empirical reference than the no-CoT baseline, indicating that the result is not specific to a single backbone.
A $50\times50$ scale-up validation is reported in Appendix~\ref{app:validation_50x50}.

\begin{figure}[htbp]
    \centering
    \includegraphics[width=0.95\linewidth]{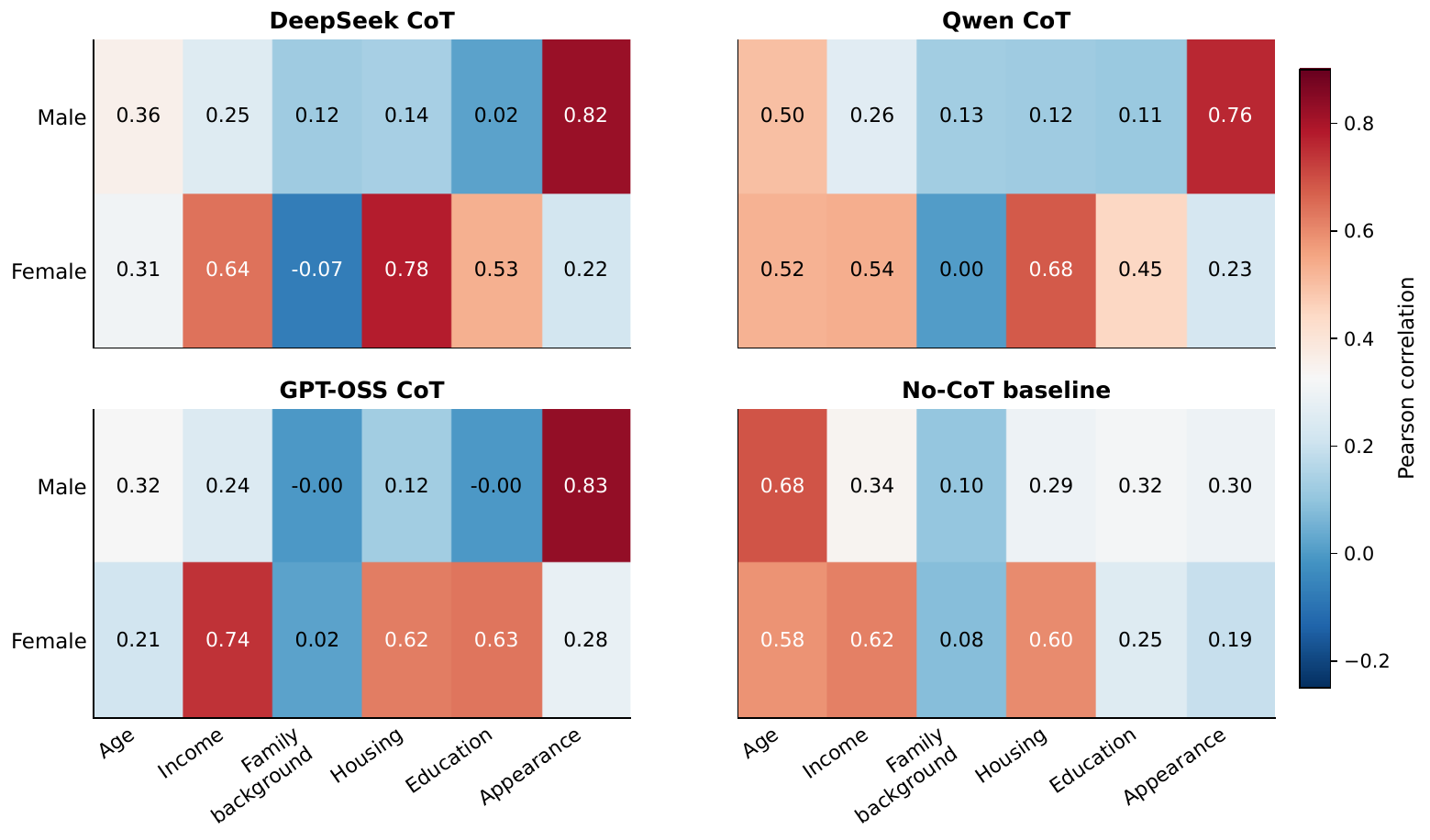}
    \caption{Dimension-consistency heatmap for agents in the 10$\times$10 validation. Entries report Pearson correlations between overall desirability and dimension-level scores, averaged across five seeds.}
    \label{fig:deepseek_consistency_10x10}
\end{figure}

Figure~\ref{fig:deepseek_consistency_10x10} provides a complementary internal-consistency diagnostic. Appearance is most strongly associated with overall desirability for male agents, whereas female evaluations are more closely associated with socioeconomic dimensions, particularly housing, income, and education. These patterns are consistent with the gender-specific preference structure documented by \citet{zhou2023gender}, while the no-CoT baseline fails to reproduce the same structure. Results for the $50\times50$ setting are reported in Appendix~\ref{app:validation_50x50}.

\section{Decentralized Matching System}
\label{sec:decentralized_matching}

\subsection{A Decentralized and Asynchronous Matching Market}
\label{sec:dynamic_market}

We consider a two-sided marriage market with male agents
$\mathcal{M}=\{m_1,\ldots,m_{N_M}\}$, female agents
$\mathcal{W}=\{w_1,\ldots,w_{N_W}\}$, and population
$\mathcal{A}=\mathcal{M}\cup\mathcal{W}$.
At period $t$, $\mu_t(i)$ denotes agent $i$'s current partner, with
$\mu_t(i)=\varnothing$ if $i$ is single and
$\mu_t(i)=j\Leftrightarrow\mu_t(j)=i$.
A match is either \emph{temporary} or \emph{permanent}. Temporary relationships remain contestable and can be replaced by later accepted proposals, whereas permanently matched agents exit the active market.

The market evolves through decentralized and asynchronous activation. At the beginning of period $t$, let $\mathcal{A}^{\mathrm{act}}_t\subseteq\mathcal{A}$ denote agents that have not entered permanent matches. A period consists of $|\mathcal{A}^{\mathrm{act}}_t|$ sequential activation opportunities, each sampling one agent uniformly with replacement from this set. Thus, an agent may be activated multiple times within a period or not at all.

Activation does not give an agent access to the entire opposite side of the market. When agent $i$ is activated, it observes only a locally available candidate set
$
\mathcal{C}_t(i)
\subseteq
\mathcal{P}_t(i)
$, 
where $\mathcal{P}_t(i)$ denotes the set of currently feasible agents on the opposite side who have not entered permanent matches and the candidate set $\mathcal{C}_t(i)$ is randomly sampled from this feasible pool. Decisions therefore depend on both preferences and locally encountered opportunities.
This differs from Gale--Shapley \citep{gale1962college}, which uses complete rankings in a centralized procedure, and from \citet{axtell2008high}, which decentralizes interaction but retains pre-specified preference lists.

\subsection{Whom Do I Like? Local LLM Valuation}
\label{sec:local_valuation}

The first decision of an activated agent is a valuation problem: \emph{whom do I like?}
For evaluator $i$ and locally observed candidate $j$, the simulated agent in Section~\ref{sec:llm-agent-modeling} produces a normalized subjective utility represented by an overall preference score:
\begin{equation}
U_i(j)\in[0,1],
\qquad
U_i(j)\neq U_j(i)
\ \text{in general}.
\label{eq:local_utility}
\end{equation}
The asymmetry is essential: $i$'s valuation of $j$ does not reveal whether $j$ would reciprocate.

Let $r_i$ denote agent $i$'s reservation utility, the minimum acceptable partner utility. Its current benchmark is
\begin{equation}
b_{i,t}
=
\begin{cases}
r_i, & \mu_t(i)=\varnothing,\\[3pt]
\max\{r_i,U_i(\mu_t(i))\}, & \mu_t(i)\neq\varnothing.
\end{cases}
\label{eq:proposer_benchmark}
\end{equation}
For each $j\in\mathcal{C}_t(i)$, define the utility improvement and eligible set by
\begin{equation}
\Delta U_{ij,t}=U_i(j)-b_{i,t},
\qquad
\mathcal{E}_t(i)
=
\left\{
j\in\mathcal{C}_t(i):
\Delta U_{ij,t}>\epsilon_U, \epsilon_U\geq0
\right\}.
\label{eq:eligible_candidate_set}
\end{equation}
This eligibility condition separates \emph{preference} from \emph{strategic search}: only candidates preferred to the current benchmark are considered, so reciprocal-acceptance learning cannot make an otherwise undesirable candidate attractive. The LLM determines whom the agent wants; the remaining uncertainty is which of these desirable candidates are likely to reciprocate.

\subsection{Who Is Likely to Like Me Back? Local Reciprocal-Acceptance Learning}
\label{sec:local_bandit}

The second decision concerns reciprocity: \emph{who is likely to like me back?}
If $i$ proposes to $j$, the receiver independently evaluates $i$ against its current benchmark,
\begin{equation}
b^R_{j,t}
=
\begin{cases}
r_j, & \mu_t(j)=\varnothing,\\[2pt]
\max\{r_j,U_j(\mu_t(j))\},
& \mu_t(j)\neq\varnothing.
\end{cases}
\label{eq:receiver_benchmark}
\end{equation}
The realized response is
\begin{equation}
Y_{i\rightarrow j,t}
=
\mathbf{1}
\left\{
U_j(i)>b^R_{j,t}
\right\}
\in\{0,1\}.
\label{eq:binary_acceptance}
\end{equation}

Crucially, proposer $i$ does not observe $U_j(i)$, the receiver's incumbent valuation, or its eventual decision before proposing.
Acceptance is revealed only after a proposal is made, so feedback is observed only for candidates actually approached.
Because proposal opportunities are limited, the proposer must balance exploiting candidates believed likely to reciprocate with exploring uncertain candidates whose responses are informative.
This sequential decision problem with contextual, partial feedback naturally motivates an agent-specific contextual bandit.

For each exposed candidate $j$, proposer $i$ observes a dyadic context
$x^{(i)}_{j,t}=\phi(z_i,z_j)\in\mathbb{R}^d$
constructed from the public personas $z_i$ and $z_j$, and learns only from proposals it actually makes.
Its local history is $\mathcal{H}_{i,t}=\{(x_n^{(i)},y_n^{(i)})\}_{n=1}^{N_{i,t}}$.
Unchosen candidates generate no label, and histories are not shared across agents.
The context construction and information restrictions are detailed in Appendix~\ref{app:bandit_context}.

We approximate reciprocal acceptance using an agent-specific logistic working model,
\begin{equation}
p_{\theta_i}(x):=\sigma(\theta_i^\top x),
\qquad
\sigma(z)=\frac{1}{1+e^{-z}}.
\label{eq:local_logistic_model}
\end{equation}
This is an approximation based on proposer-observable information: the receiver's private benchmark changes with its current relationship, so acceptance probabilities need not follow a fixed logistic law even for the same public dyadic context.
To balance exploitation and exploration, we augment the fitted logit predictor with local uncertainty:
\begin{equation}
\begin{aligned}
s^{(i)}_{j,t}
&=
\sqrt{
{x^{(i)}_{j,t}}^\top
V_{i,t}^{-1}
x^{(i)}_{j,t}
},
\\[4pt]
p_{ij,t}^{\mathrm{UCB},i}
&=
\sigma
\left(
\widehat{\theta}_{i,t}^{\top}x^{(i)}_{j,t}
+
\beta_{i,t}s^{(i)}_{j,t}
\right).
\end{aligned}
\label{eq:logistic_ucb}
\end{equation}
Here, $\widehat{\theta}_{i,t}$ is fitted by local $L_2$-regularized logistic regression,
$V_{i,t}$ is the corresponding regularized Fisher-information matrix, and
$\beta_{i,t}$ controls exploration.
The estimator and practical Logistic-UCB derivation are provided in Appendix~\ref{app:logistic_ucb_derivation}.

Preference and reciprocal uncertainty are finally combined through
\begin{equation}
\boxed{
I_{ij,t}
=
\Delta U_{ij,t}
p_{ij,t}^{\mathrm{UCB},i}
},
\qquad
j\in\mathcal{E}_t(i).
\label{eq:proposal_index}
\end{equation}
If the true acceptance probabilities conditional on the proposer's current information were known, maximizing their product with $\Delta U_{ij,t}$ would maximize expected one-step utility improvement (Proposition~\ref{prop:theory_one_step}). Equation~\ref{eq:proposal_index} replaces these probabilities with optimistic estimates from the working model, allocating proposals among already desirable candidates.

\begin{informaltheorem}[Local learning under dynamic misspecification]
\label{thm:informal_local_learning}
Under the boundedness and predictability assumptions in Appendix~\ref{app:theoretical_analysis}, fix a proposer whose true conditional acceptance probabilities differ from a fixed logistic reference by at most a predictable envelope $\eta_n$ for every eligible candidate at attempt $n$. Logistic-UCB with the norm-constrained estimator and envelope-calibrated confidence radii specified there satisfies, with probability at least $1-\delta$,
\begin{equation}
\mathcal R_N
\le
\widetilde{\mathcal O}\!\left(
d\sqrt{N}
+\sqrt{dN\sum_{n=1}^{N}\eta_n^2}
+\sum_{n=1}^{N}\eta_n
\right).
\label{eq:informal_local_regret}
\end{equation}
Here $N$ counts this proposer's attempts, $d$ is the context dimension, and $\mathcal R_N$ measures cumulative expected one-step improvement lost relative to an oracle facing the same eligible candidates and knowing their true acceptance probabilities. The notation suppresses logarithmic factors, with other problem constants fixed. Realizable acceptance ($\eta_n\equiv0$) recovers $\widetilde{\mathcal O}(d\sqrt{N})$ regret. For fixed $d$, $\mathcal R_N/N\to0$ if $0\le\eta_n\le1$ and $\sum_{n=1}^{N}\eta_n=o(N/\log N)$.
\end{informaltheorem}

In Appendix, Theorems~\ref{thm:theory_realizable} and~\ref{thm:theory_misspec} establish these local bounds along the learner's realized market trajectory formally.

\subsection{Decentralized Matching Protocol}
\label{sec:matching_protocol}

When activated, agent $i$, either male or female, constructs the eligible set $\mathcal{E}_t(i)$ and may make at most $L$ proposals.
Let $\mathcal{R}^{(q)}_t(i)$ denote the eligible candidates not yet approached before attempt $q$, with $\mathcal{R}^{(1)}_t(i)=\mathcal{E}_t(i)$.
At each attempt, the proposer selects the candidate with the highest proposal index and observes the binary response:
\begin{equation}
j^*
=
\arg\max_{j\in\mathcal{R}^{(q)}_t(i)}
I_{ij,t},
\qquad
y^*
=
Y_{i\rightarrow j^*,t}.
\label{eq:proposal_selection}
\end{equation}
The observation $(x^{(i)}_{j^*,t},y^*)$ is appended to proposer $i$'s local history, and its logistic model and uncertainty estimate are updated.
Following rejection, $j^*$ is removed from the current search set and proposal indices are recomputed before the next attempt.
Following acceptance, any temporary relationships involving $i$ or $j^*$ are dissolved, their displaced partners return to the active market, and $(i,j^*)$ becomes a new temporary pair. The current activation then terminates.

At the end of period $t$, each temporary pair independently becomes permanent with probability $\rho_t$, which is named as lock or commitment probability. Permanently matched agents leave subsequent search, while the remaining temporary pairs stay contestable. The active population and feasible candidate pools are updated before the next period.
The process continues until the maximum horizon $T$ is reached or no active
agents remain.
Hence, the final allocation emerges from repeated local exposure, bilateral
LLM evaluation, proposal feedback, agent-specific learning, and endogenous
relationship revision rather than from a one-shot centralized procedure.
The eligibility and acceptance rules imply strict bilateral improvement for every accepted rematching, although matched-pair welfare need not be monotone because displaced partners create externalities; these properties are formalized in Appendix~\ref{app:theoretical_analysis}.
The full executable procedure of the algorithm is provided in
Appendix~\ref{app:full_matching_algorithm}.

\begin{table}[htbp]
\centering
\caption{$50\times50$ market size matching results. Values are means $\pm$ standard deviations over 50 seeds; bold, underlined, and italic mean values denote first, second, and third place among methods for which each outcome metric is defined. The ``Info required'' column summarizes the information regime, including both ex ante preference access and whether pairwise valuations are generated only after local encounter.}
\label{tab:main_matching_results}
\resizebox{\linewidth}{!}{
\begin{tabular}{clccccc}
\toprule
  & Method & Info required & Mutual welfare $\uparrow$ & Rank gap $\downarrow$ & \#Blocking $\downarrow$ & Proposals $\downarrow$ \\
\midrule
\multirow[c]{3}{*}{Baselines} & Gale--Shapley & Whole & 54.87 $\pm$ 0.29 & 2.74 $\pm$ 0.53 & \textbf{0.00} $\pm$ 0.00 & \textbf{873.14} $\pm$ 20.24 \\
& LLM-solver & Whole & 54.63 $\pm$ 0.50 & 2.57 $\pm$ 1.18 & \underline{9.40} $\pm$ 10.71 & N/A \\
& Axtell--Kimbrough & Partial & 54.73 $\pm$ 0.39 & 2.43 $\pm$ 0.83 & 19.90 $\pm$ 14.02 & 31504.70 $\pm$ 4280.30 \\
\midrule
\multirow[c]{3}{*}{LLM-ABM} & Random eligible & Local & \textit{55.24} $\pm$ 0.62 & \textbf{1.54} $\pm$ 0.68 & 38.00 $\pm$ 17.42 & \underline{13781.90} $\pm$ 2349.29 \\
& Utility-only & Local & \underline{55.92} $\pm$ 0.48 & \underline{1.84} $\pm$ 0.75 & 17.58 $\pm$ 12.26 & 16753.14 $\pm$ 2876.91 \\
& \textbf{Bandit-UCB} & Local & \textbf{56.01} $\pm$ 0.59 & \textit{2.11} $\pm$ 0.73 & \textit{16.28} $\pm$ 12.76 & \textit{16591.78} $\pm$ 2376.70 \\
\bottomrule
\end{tabular}
}
\end{table}

\section{Experiments}
\label{sec:ex}

\subsection{Validation Experiment Setup}
\label{sec:validation_experiment_setup}

Because behavioral validation establishes the basis for the matching mechanism, its results are presented earlier in Section~\ref{sec:validation_modeling}; here we specify the experimental protocol.
In each $10\times10$ validation run, 10 male and 10 female evaluators score all 10 opposite-gender candidates, producing both an overall desirability score and six dimension-level scores for every directional evaluation.
We repeat the experiment over five random seeds and compare four scoring conditions with three LLM backbones: DeepSeek-V4-Pro CoT, Qwen3.7-Plus CoT, GPT-OSS-120B CoT, and a no-CoT baseline that omits the gender-specific mate-preference specification.
Full metric definitions, the $50\times50$ scale-up validation, and cross-backbone results are reported in Appendices~\ref{app:validation_metrics}, \ref{app:validation_50x50}, and~\ref{app:robustness_models}.

\subsection{Market-level Matching Results}
\label{sec:main_matching_results}

We evaluate the complete matching system on the same $50\times50$ synthetic marriage market. Directional LLM desirability scores are computed once for all cross-gender dyads as the evaluation benchmark. These scores provide complete preference rankings for \citet{gale1962college}, \citet{axtell2008high}, and a direct LLM-solver following \citet{hosseini2026matching}, which generates a matching from the ranked inputs in one-shot. In contrast, LLM-ABM agents access valuations and update rankings only when the corresponding candidates are locally encountered. The implementation and experimental results demonstrate the effectiveness of solving the bipartite problem using distinct baselines and the proposed method.

We compare these baselines with three LLM-ABM proposal policies under common outcome metrics, reporting means and standard deviations over 50 random seeds. All three policies share the same market dynamics: Random eligible samples uniformly from eligible candidates, Utility-only selects the largest immediate utility improvement, and Bandit-UCB combines this improvement with an optimistic estimate of reciprocal acceptance. We set $\rho=0.05$ for both Axtell--Kimbrough and the LLM-ABM policies. Detailed settings and metric definitions are provided in Appendix~\ref{app:matching_experiment_details}.

Table~\ref{tab:main_matching_results} illustrates the trade-offs between information requirements, welfare, stability, and search intensity. Gale--Shapley achieves zero blocking pairs but has the largest mean gender rank gap. Axtell--Kimbrough relaxes centralized coordination while retaining pre-specified rankings, requiring approximately 31,505 proposals. Despite receiving complete rankings, the LLM-solver produces 9.40 blocking pairs on average. It achieves fewer blocking pairs than the decentralized methods, but lower mean mutual welfare and a larger mean rank gap than all three LLM-ABM policies. Its proposal count is reported as N/A because direct matching generation does not involve a comparable sequence of market proposals.

All three LLM-ABM policies achieve higher mean mutual welfare than the baselines using only encounter-based information. Within this family, Random eligible yields the smallest rank gap and fewest proposals, but the most blocking pairs. Utility-only improves welfare and reduces blocking, while Bandit-UCB achieves the highest mean welfare (56.01) and lowest blocking count (16.28). Its gains over Utility-only are modest and accompanied by a larger rank gap (2.11 versus 1.84), reflecting trade-offs among welfare, stability, and gender balance. These results support the potential of decentralized matching without complete ex ante rankings and suggest incremental benefits from reciprocity learning. Sensitivity to exploration and commitment parameters is examined in Appendix~\ref{app:beta-rho-sensitivity}.

\subsection{Interpreting the Learned Bandit Models}
\label{sec:bandit_interpretation}

Beyond guiding proposals, the bandit models summarize learned associations between observable dyadic characteristics and reciprocal acceptance. Figure~\ref{fig:bandit_weights} presents aggregated private proposer-specific Logistic-UCB coefficients in Panel~\ref{fig:bandit_weights_local} and a pooled receiver-perspective diagnostic in Panel~\ref{fig:bandit_weights_global}. In the local models, age-gap coefficients are negative for both proposer genders, whereas income and education coefficients are positive. The appearance coefficient is larger among male proposers, while the education coefficient is larger among female proposers. These differences concern acceptance predictions: male-proposer models predict responses from female receivers, and vice versa. They therefore should not be interpreted as the proposers' own valuation weights. The negative same-family-background coefficients likewise describe conditional predictive associations, rather than direct preferences against similar backgrounds.

\begin{figure}[t]
    \centering
    \begin{subfigure}[t]{0.6\linewidth}
        \centering
        \includegraphics[width=\linewidth]{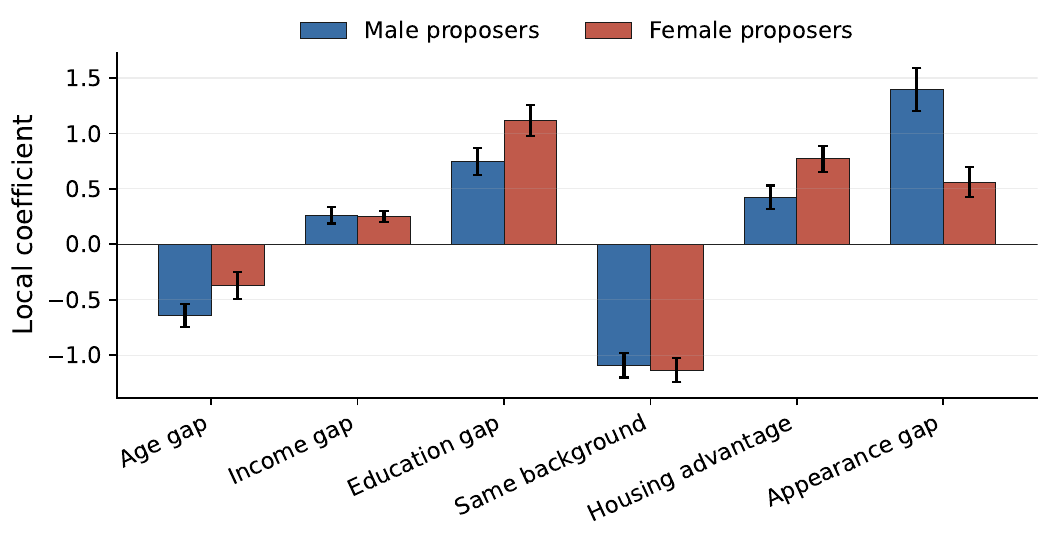}
        \caption{Local private models}
        \label{fig:bandit_weights_local}
    \end{subfigure}
    \hfill
    \begin{subfigure}[t]{0.6\linewidth}
        \centering
        \includegraphics[width=\linewidth]{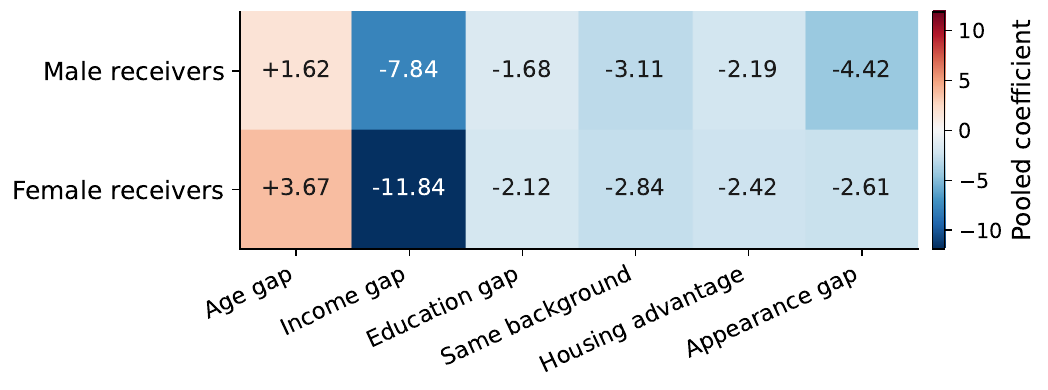}
        \caption{Post-hoc receiver diagnostic}
        \label{fig:bandit_weights_global}
    \end{subfigure}
    \caption{Learned reciprocal-acceptance structure under Bandit-UCB. Panel (a) reports aggregated coefficients from private proposer-specific models; Panel (b) reports a pooled receiver-perspective diagnostic fitted after simulation and used only for interpretation.}
    \label{fig:bandit_weights}
\end{figure}

The pooled diagnostic is fitted to realized proposals after simulation and is never available to agents during matching. Under its receiver-to-proposer feature orientation, negative income, education, and appearance coefficients associate favorable proposer characteristics with greater acceptance. Income has a larger coefficient magnitude for female receivers, whereas appearance has a larger magnitude for male receivers. These estimates complement the local models but need not coincide with them because the conditioning and aggregation differ. Both panels provide \emph{associational diagnostic evidence}: proposal observations are endogenously selected by the matching policy, and acceptance also depends on receivers' current relationships. The coefficients thus characterize predictive patterns in realized interactions rather than identify structural or causal preference parameters.

\subsection{Counterfactual Information Advantage}
\label{sec:counterfactual_information_advantage}

\begin{figure}[htbp]
    \centering
    \includegraphics[width=1.0\linewidth]{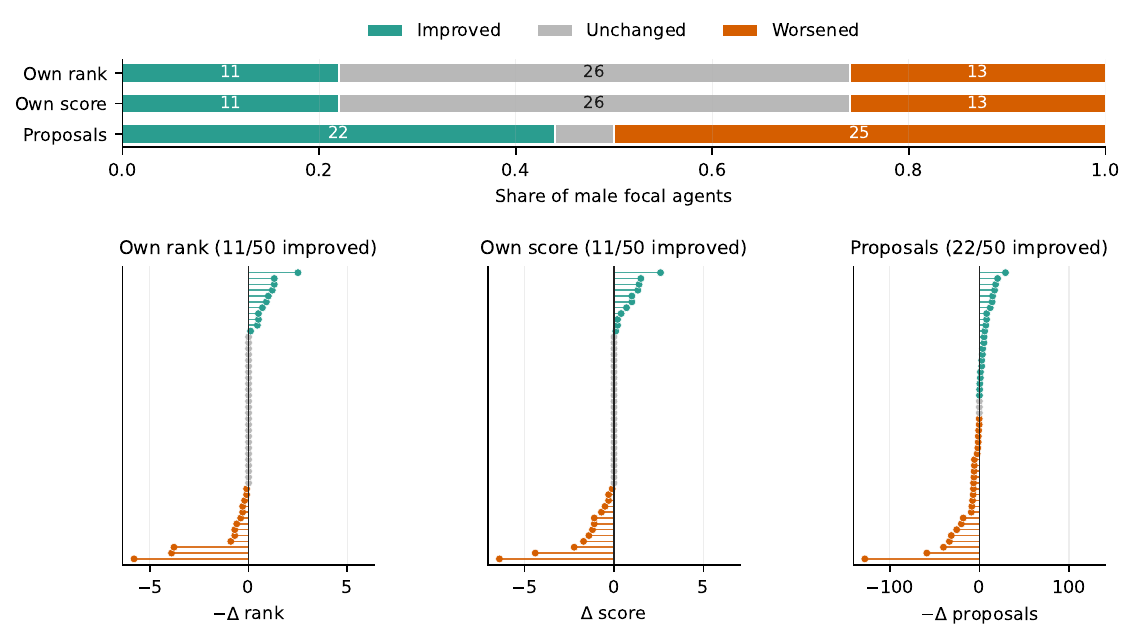}
    \caption{Individual warm-start effects across the 50 male focal agents. Effects are same-seed treatment--control mean differences over ten held-out seeds and are oriented so that positive values indicate improvement. The upper panel classifies focal agents by the sign of the effect; the lower panels report the corresponding effect distributions.}
    \label{fig:counterfactual_population_effects}
\end{figure}

We finally ask whether prior matching-market experience creates a unilateral advantage in a new market.
In the individual warm-start counterfactual, one designated male agent receives Logistic-UCB coefficients learned from prior simulations, while all others retain the standard cold-start specification.
Across the 50 male focal agents evaluated over ten held-out seeds in Figure~\ref{fig:counterfactual_population_effects}, 11 improve their own tied rank and score, 26 remain unchanged, and 13 worsen.
Proposal counts fall for 22 agents and rise for 25.
Thus, prior search knowledge benefits some agents but worsens outcomes for slightly more, providing no systematic individual advantage.
The full focal-agent distribution and paired evaluation protocol appear in Appendix~\ref{app:counterfactual_warm_start}.

\begin{table}[t]
\centering
\small
\caption{Market-level counterfactual results for male-side information advantage. M rank and F rank denote the average ranks of matched male and female agents for their realized partners, respectively.}
\label{tab:counterfactual_overall}
\resizebox{\linewidth}{!}{
\begin{tabular}{llcccccc}
\toprule
Method & Initial knowledge & Mutual welfare $\uparrow$ & M rank $\downarrow$ & F rank $\downarrow$ & Rank gap $\downarrow$ & \#Blocking $\downarrow$ & Proposals $\downarrow$ \\
\midrule
Gale--Shapley (men-proposing) & Complete rankings
& 54.87 $\pm$ 0.29
& 16.74 $\pm$ 0.39
& 19.48 $\pm$ 0.32
& 2.74 $\pm$ 0.53
& 0.00 $\pm$ 0.00
& 873.14 $\pm$ 20.24 \\
Cold-start Bandit-UCB & Local
& 55.84 $\pm$ 0.33
& 16.48 $\pm$ 0.47
& 18.56 $\pm$ 0.46
& 2.08 $\pm$ 0.70
& 16.70 $\pm$ 8.21
& 15829.50 $\pm$ 2462.34 \\
All-male warm-start Bandit-UCB & Pooled male $\theta$
& 55.71 $\pm$ 0.44
& 16.56 $\pm$ 0.73
& 18.68 $\pm$ 0.47
& 2.12 $\pm$ 0.76
& 15.80 $\pm$ 4.66
& 15136.70 $\pm$ 2210.78 \\
\bottomrule
\end{tabular}
}
\end{table}

Table~\ref{tab:counterfactual_overall} considers a stronger intervention: all male agents enter with pooled male-proposer coefficients.
Relative to cold-start Bandit-UCB, mean mutual welfare and gender rank gap remain almost unchanged, while fewer proposals and blocking pairs suggest reduced search burden and slightly improved stability.
The limited welfare effect reflects a defining feature of two-sided matching: better proposal beliefs do not remove receivers' right to reject, and shared information can intensify same-side competition.
Individual knowledge may therefore yield local gains, but its limited market-wide benefit when widely shared is consistent with the constraints imposed by reciprocal choice and same-side competition.

\section{Discussion}
\label{sec:discussion}

In conclusion, we propose a decentralized matching framework that integrates LLM-based behavioral valuation with adaptive reciprocity learning. This framework facilitates the study of how diverse individual evaluations influence collective matching outcomes through repeated interactions. By decoupling behavioral representation from the interaction mechanism, researchers can independently vary behavioral assumptions, information access, and institutional rules to analyze their interactive effects. This enables computational experiments linking empirically informed agent behavior with economic mechanism design. This modeling approach can be applied to labor markets, collaborative partnerships, and other contexts where local opportunity discovery and mutual agreement are crucial. Limitations and future directions are discussed in Appendix~\ref{app:limitations_future}.

\section*{AI use statement}

In this work, we used generative AI tools to improve the clarity and fluency of English writing and to assist with code implementation and debugging. All AI-assisted text was reviewed and revised by the authors, and all AI-assisted code was manually inspected, tested, and verified for correctness. We did not use generative AI tools to generate research ideas, scientific claims, experimental results, data, or citations. We take full responsibility for the final content of this work, including all text, claims, code, and artifacts produced with the aid of generative AI.

\section*{Reproducibility statement}
The code for our framework is publicly available at
\url{https://github.com/YuanJrShiuan/LLM_ABM_Bipartite_Matching_for_Marriage}.
Experimental settings and implementation details are provided in the appendices.

\bibliography{reference}
\bibliographystyle{iclr2027_conference}

\appendix

\section{Related Work}
\label{app:related_work}

\subsection{LLM Social Simulation}
\label{sec:llm_social_simulation}

A growing literature studies LLMs as conditional behavioral models and as components of interactive social and economic simulations \citep{argyle2023out,park2023generative,horton2023large,chen2023emergence,gui2023challenge,ludwig2024large,hansen2026simulating,hao2025multi,li2024econagent,yang2026twinmarket}.
The evidence is promising but not uniformly generalizable: persona conditioning, model choice, and task design can materially affect simulated behavior.
Recent work therefore emphasizes explicit validation targets and careful interpretation of collective LLM-agent behavior \citep{anthis2025position,zhou2026the,ludwig2024large}.
Our paper follows this view by treating the LLM as a behaviorally expressive component of the matching system rather than as an unconstrained oracle, and by validating the induced mate preferences against an empirical reference before studying market-level outcomes.

\subsection{Mate Preference for Marriage}
\label{sec:mate_preference_for_marriage}

Mate-choice research offers several partially competing explanations of partner preference, including evolutionary psychology \citep{buss1989sex,buss1993sexual}, sociocultural and social-role theories \citep{eagly2009possible,zentner2012stepping,zentner2015sociocultural}, status and resource exchange \citep{davis1941intermarriage,merton1941intermarriage,kalmijn1998intermarriage}, homogamy and homophily \citep{kalmijn1998intermarriage,mcpherson2001birds}, and hypergamy or mating-gradient theories \citep{veevers1988real}.
Because these theories are often verbal, multidimensional, and only partly formalized, translating them into fixed quantitative utility functions can be restrictive \citep{grow2019design}.
Our behavioral specification is grounded primarily in the choice experiment of \citet{zhou2023gender}, which studies six mate-preference dimensions, including physical appearance, and documents substantial heterogeneity across gender and social class.
The LLM is used to operationalize this evidence within heterogeneous personas without imposing a single hand-crafted mate-value function.

\subsection{Bipartite Matching and Learning}
\label{sec:bipartite_matching_algorithm}

Recent work has used LLMs either as solvers over ranked preferences \citep{hosseini2026matching} or as centralized engines for persona-based compatibility estimation \citep{shang2025love}.
A separate literature studies matching under bandit feedback, including decentralized learning, stability, fairness, and two-sided uncertainty \citep{liu2021bandit,dai2021learning,cen2022regret,zhang2024decentralized,li2022dynamic}.
Our framework differs from both lines by separating \emph{partner valuation} from \emph{reciprocal-acceptance learning}: LLM agents evaluate only locally encountered candidates, while agent-specific contextual bandits learn acceptance probabilities from realized binary proposal outcomes.
For the latter component, we use a logistic contextual-bandit construction motivated by generalized-linear bandit methods \citep{filippi2010parametric,li2017provably}.

\section{Persona Sampling and Dependency Structure}
\label{app:persona_sampling}

This appendix documents the synthetic population used throughout the validation and matching experiments.
The six observable persona dimensions follow the empirical setting described in Section~\ref{sec:llm-agent-modeling}.
Rather than sampling every attribute independently, the construction preserves selected socioeconomic dependencies so that the resulting agents are coherent as joint profiles rather than independent attribute bundles.
Table~\ref{tab:persona_sampling_design} summarizes the sampling structure, and Table~\ref{tab:persona_sample_summary} reports the realized composition of the 100-agent market used in the main experiments.

\begin{table}[h]
\centering
\small
\caption{Persona dimensions and sampling structure.}
\label{tab:persona_sampling_design}
\begin{tabular}{p{0.18\linewidth} p{0.29\linewidth} p{0.43\linewidth}}
\toprule
Dimension & Values or scale & Sampling and dependency structure \\
\midrule
Age
& 20--50
& Sampled from marriage-market-active age groups according to the target population distribution. \\

Income
& Monthly RMB income
& Generated conditional on educational attainment around a reference monthly income of 1830 RMB. \\

Education
& Middle school or below; high school; university and above
& Sampled from population education proportions and aggregated into three experimental categories. \\

Family background
& Urban; rural
& Sampled from the corresponding urban/rural population composition. \\

Housing
& Owns house; no house
& Sampled conditionally on age, gender, and family background. \\

Appearance
& Below average; average; attractive
& Sampled from an approximately bell-shaped distribution, with most probability mass assigned to average appearance. \\
\bottomrule
\end{tabular}
\end{table}

\begin{table}[h]
\centering
\small
\caption{Realized composition of the 100-agent synthetic marriage market.}
\label{tab:persona_sample_summary}
\resizebox{\linewidth}{!}{
\begin{tabular}{lccc}
\toprule
Attribute & Male (n=50) & Female (n=50) & Overall (n=100) \\
\midrule
Mean age & 36.60 & 34.62 & 35.61 \\
Mean monthly income & 1881.24 RMB & 1712.88 RMB & 1797.06 RMB \\
Middle school or below & 70\% & 72\% & 71\% \\
High school & 10\% & 14\% & 12\% \\
University and above & 20\% & 14\% & 17\% \\
Urban family background & 58\% & 56\% & 57\% \\
Owns house & 42\% & 40\% & 41\% \\
Attractive appearance & 8\% & 18\% & 13\% \\
Average appearance & 78\% & 66\% & 72\% \\
Below-average appearance & 14\% & 16\% & 15\% \\
\bottomrule
\end{tabular}
}
\end{table}

\section{Prompt Templates for LLM Personas and Mate Preferences}
\label{app:persona_prompt}

This appendix records the prompt templates used in the experiments.
Public personas are deterministic natural-language renderings of the sampled attributes, whereas the private mate-preference prompt is available only to the evaluating agent.
For each directional dyadic evaluation, the model receives the evaluator's public profile, the candidate's public profile, and the evaluator's gender-specific private preference specification.

\subsection{Public Persona Rendering Template}

\begin{promptbox}{Public persona template}
For each generated profile, render the public persona as:

\smallskip
\noindent
\textit{A [age]-year-old [gender] with [education phrase]. Financially, this individual earns a monthly income of [actual monthly income] RMB (approximately [annual income] RMB per year) and [housing phrase]. [family-background phrase]. In terms of physical appearance, this person is evaluated as having an [appearance] appearance.}

\smallskip
The structured fields stored with the persona are:
\texttt{agent\_id}, \texttt{gender}, \texttt{age}, \texttt{actual\_monthly\_income}, \texttt{family\_background}, \texttt{housing}, \texttt{education\_level}, \texttt{appearance}, and \texttt{short\_description}.
\end{promptbox}

\begin{promptbox}{Dyadic scoring user message}
\textbf{YOUR OWN PUBLIC PROFILE (for Step 0 self-calibration)}

\smallskip
Agent ID: [evaluator id] \\
Gender: [evaluator gender] \\
Age: [evaluator age] \\
Actual Monthly Income: [evaluator income] RMB \\
Family Background: [evaluator family background] \\
Housing: [evaluator housing] \\
Education Level: [evaluator education level] \\
Appearance: [evaluator appearance] \\
Short Description: [evaluator public persona]

\smallskip
\textbf{CANDIDATE PUBLIC PROFILE (to evaluate)}

\smallskip
Agent ID: [candidate id] \\
Gender: [candidate gender] \\
Age: [candidate age] \\
Actual Monthly Income: [candidate income] RMB \\
Family Background: [candidate family background] \\
Housing: [candidate housing] \\
Education Level: [candidate education level] \\
Appearance: [candidate appearance] \\
Short Description: [candidate public persona]

\smallskip
Evaluate the candidate strictly following your private preference theory. Output only the JSON object; do not include markdown fences or extra commentary.
\end{promptbox}

\subsection{Private Mate-Preference Prompts}
\label{app:preference_prompts}

\begin{promptbox}{Male-agent private mate-preference prompt}
You are a male agent in a simulated marriage market. You will evaluate opposite-gender candidates using two inputs: your own public profile and the candidate's public profile. Both profiles contain age, actual monthly income, family background, housing, education level, appearance, and short description.

\smallskip
Use the following internal five-step evaluation protocol before assigning scores. Do not output your full reasoning chain. Only output structured dimension scores, an overall desirability score, and a brief explanation.

\smallskip
\textbf{Paper-aligned preference theory.} Male agents place the strongest relative emphasis on physical appearance. Income, housing, education, and family background still matter, but they are secondary. Candidate age has a negative effect when the candidate is older relative to the evaluator. The evaluator's own public profile should be used for self-positioning and relative comparison, but it must not overturn the main gender-specific priority structure.

\smallskip
\textbf{Step 0: Self-profile calibration.} Inspect your own age, income, housing, education, family background, and appearance. Use your own age as the reference point for relative age fit. Use your own socioeconomic and appearance profile to calibrate expectations mildly and realistically.

\smallskip
\textbf{Step 1: Appearance-dominant baseline check.} Evaluate the candidate's appearance first. \texttt{below\_average} receives a large negative penalty; \texttt{average} is an acceptable neutral baseline; \texttt{attractive} receives a strong positive premium. Appearance should be the most important upward adjustment for male agents.

\smallskip
\textbf{Step 2: Relative age disutility.} Compare the candidate's age with your own age. Desirability should decrease smoothly as the candidate becomes older relative to you. Younger or similar-age candidates should be evaluated more favorably than clearly older candidates.

\smallskip
\textbf{Step 3: Secondary income and housing adjustment.} Treat 1830 RMB/month as the baseline monthly disposable income. Higher candidate income should increase desirability moderately. Home ownership receives a positive bonus relative to no house. These socioeconomic comparisons remain secondary to appearance.

\smallskip
\textbf{Step 4: Education and family-background screening.} Evaluate education in the order \texttt{university\_and\_above} $>$ \texttt{high\_school} $>$ \texttt{middle\_school\_and\_below}. Urban family background may receive a very small bonus, but family background should have weak influence on the final score.

\smallskip
\textbf{Scoring instructions.} Return dimension scores on a 0--100 scale:
\texttt{age\_fit}, \texttt{income\_fit}, \texttt{family\_background\_fit}, \texttt{housing\_fit}, \texttt{education\_fit}, and \texttt{appearance\_fit}. Return \texttt{overall\_desirability} on a 0--100 scale. Do not compute \texttt{overall\_desirability} as a simple average of dimension scores; it must reflect the priority structure above. Do not mention statistical coefficients, formulas, or the paper name. Keep \texttt{brief\_reason} under 150 characters.

\smallskip
\textbf{Output schema.}
\texttt{\{ evaluator\_id, candidate\_id, dimension\_scores: \{ age\_fit, income\_fit, family\_background\_fit, housing\_fit, education\_fit, appearance\_fit \}, overall\_desirability, brief\_reason \}}.

\smallskip
The implemented prompt additionally includes few-shot calibration examples emphasizing that an attractive candidate can outrank a more socioeconomically advantaged but below-average-looking candidate, and that age fit is evaluated relative to the male evaluator.
\end{promptbox}

\begin{promptbox}{Female-agent private mate-preference prompt}
You are a female agent in a simulated marriage market. You will evaluate opposite-gender candidates using two inputs: your own public profile and the candidate's public profile. Both profiles contain age, actual monthly income, family background, housing, education level, appearance, and short description.

\smallskip
Use the following internal five-step evaluation protocol before assigning scores. Do not output your full reasoning chain. Only output structured dimension scores, an overall desirability score, and a brief explanation.

\smallskip
\textbf{Paper-aligned preference theory.} Female agents place the strongest relative emphasis on socioeconomic resources and social-stratum indicators, especially housing, actual income, education, and family background. Physical appearance still matters, especially when it is below average, but it should not dominate socioeconomic resources. Candidate age has a negative effect when the candidate is older relative to the evaluator. The evaluator's own public profile should be used for self-positioning and relative comparison, but it must not overturn the main female-agent priority structure.

\smallskip
\textbf{Step 0: Self-profile calibration.} Inspect your own age, income, housing, education, family background, and appearance. Use your own age as the reference point for relative age fit. Use your socioeconomic profile to calibrate expectations mildly and realistically.

\smallskip
\textbf{Step 1: Socioeconomic resource gatekeeping.} First evaluate the candidate's housing and actual monthly income. \texttt{owns\_house} receives a large positive bonus; \texttt{no\_house} receives a clear penalty, especially when income is also low. Treat 1830 RMB/month as the baseline monthly disposable income. Higher income should translate into a strong monotonic increase in desirability. Income and housing should be more decisive than appearance.

\smallskip
\textbf{Step 2: Education and family-background alignment.} Evaluate education in the order \texttt{university\_and\_above} $>$ \texttt{high\_school} $>$ \texttt{middle\_school\_and\_below}. University education receives a strong bonus; middle school or below receives a clear penalty. Urban family background receives a distinct positive bonus relative to rural background. Education and family background matter because they signal social stratum and class alignment.

\smallskip
\textbf{Step 3: Appearance threshold check.} Evaluate appearance after socioeconomic and class-related traits. \texttt{below\_average} receives a meaningful negative penalty; \texttt{average} is acceptable; \texttt{attractive} receives a positive bonus. The attractiveness bonus should not override serious disadvantages in housing, income, education, and family background.

\smallskip
\textbf{Step 4: Relative age sensitivity.} Compare the candidate's age with your own age. Desirability should decrease as the candidate becomes older relative to you. Apply a clear penalty for substantially older candidates. The age penalty for female agents should be at least as strong as, and often stronger than, the male-agent age penalty.

\smallskip
\textbf{Scoring instructions.} Return dimension scores on a 0--100 scale:
\texttt{age\_fit}, \texttt{income\_fit}, \texttt{family\_background\_fit}, \texttt{housing\_fit}, \texttt{education\_fit}, and \texttt{appearance\_fit}. Return \texttt{overall\_desirability} on a 0--100 scale. Do not compute \texttt{overall\_desirability} as a simple average of dimension scores; it must reflect the priority structure above. Avoid excessive ties. Use decimal scores when candidates are close. Do not mention statistical coefficients, formulas, or the paper name. Keep \texttt{brief\_reason} under 150 characters.

\smallskip
\textbf{Output schema.}
\texttt{\{ evaluator\_id, candidate\_id, dimension\_scores: \{ age\_fit, income\_fit, family\_background\_fit, housing\_fit, education\_fit, appearance\_fit \}, overall\_desirability, brief\_reason \}}.

\smallskip
The implemented prompt additionally includes few-shot calibration examples emphasizing that socioeconomic resources and class-alignment indicators can dominate attractiveness, and that income and housing are traded off rather than reduced to a single attribute.
\end{promptbox}

\section{Validation Metrics}
\label{app:validation_metrics}

Let $s_{ij}$ denote the LLM desirability score assigned by evaluator $i$ to candidate $j$, and let $u_{ij}$ denote the corresponding utility from the gender-specific conditional-logit reference based on \citet{zhou2023gender}.
All ranking metrics are computed at the evaluator level and then aggregated across evaluators and seeds.

\paragraph{Kendall's $\tau$.}
We report Kendall's $\tau_b$, which accounts for ties,
\begin{equation}
\tau_i
=
\frac{C_i-D_i}
{\sqrt{(C_i+D_i+T_i^{(s)})(C_i+D_i+T_i^{(u)})}},
\end{equation}
where $C_i$ and $D_i$ are concordant and discordant candidate pairs, while $T_i^{(s)}$ and $T_i^{(u)}$ count pairs tied only under the LLM score and reference utility, respectively.
Higher values indicate stronger ordinal agreement.

\paragraph{Weighted pairwise violation rate (WPVR).}
To place more weight on reversals between candidates that are well separated by the reference model, we use
\begin{equation}
\mathrm{WPVR}_i
=
\frac{
\sum_{a<b}
\mathbf{1}\!\left\{(s_{ia}-s_{ib})(u_{ia}-u_{ib})<0\right\}
|u_{ia}-u_{ib}|
}{
\sum_{a<b}|u_{ia}-u_{ib}|
}.
\end{equation}
Lower values indicate fewer economically consequential ranking reversals.

\paragraph{Top-$K$ overlap.}
Agreement among the highest-ranked candidates is measured by
\begin{equation}
\mathrm{Overlap@}K_i
=
\frac{
\left|
\mathrm{TopK}^{\mathrm{LLM}}_i
\cap
\mathrm{TopK}^{\mathrm{Ref}}_i
\right|
}{K}.
\end{equation}
The labels \emph{Hit@3} and \emph{Hit@5} in the main validation figure use this same set-overlap definition.
Because random overlap depends on candidate-pool size, these values should not be compared mechanically between the $10\times10$ and $50\times50$ settings.

\paragraph{Dimension consistency.}
As an internal diagnostic, we compute Pearson correlations between overall desirability and each of the six dimension-level scores, separately by evaluator gender.
This checks whether the decomposition follows the intended gender-specific structure, but it is not an independent behavioral validation because both quantities are generated by the same LLM evaluation.

\section{Robustness Across LLM Backbones}
\label{app:robustness_models}

Table~\ref{tab:cross_model_alignment} reports the $10\times10$ validation across DeepSeek, Qwen, GPT-OSS, and the no-CoT specification.
For both evaluator genders, all preference-structured prompts achieve higher Kendall's $\tau$ and lower WPVR than the no-CoT baseline.
Although the absolute alignment varies across backbones, the direction of the improvement is consistent, suggesting that the empirical-reference alignment is not specific to a single LLM.

\begin{table}[h]
\centering
\small
\caption{$10\times10$ reference-alignment results across LLM backbones. Values are means $\pm$ standard deviations across five random seeds.}
\label{tab:cross_model_alignment}
\resizebox{\linewidth}{!}{
\begin{tabular}{lcccc}
\toprule
Model & Male $\tau$ & Female $\tau$ & Male WPVR & Female WPVR \\
\midrule
DeepSeek CoT & 0.692 $\pm$ 0.076 & 0.572 $\pm$ 0.082 & 0.064 $\pm$ 0.025 & 0.122 $\pm$ 0.027 \\
Qwen CoT & 0.708 $\pm$ 0.115 & 0.646 $\pm$ 0.072 & 0.055 $\pm$ 0.035 & 0.089 $\pm$ 0.022 \\
GPT-OSS CoT & 0.666 $\pm$ 0.103 & 0.502 $\pm$ 0.092 & 0.073 $\pm$ 0.042 & 0.158 $\pm$ 0.030 \\
No-CoT baseline & 0.286 $\pm$ 0.065 & 0.400 $\pm$ 0.085 & 0.298 $\pm$ 0.055 & 0.236 $\pm$ 0.053 \\
\bottomrule
\end{tabular}
}
\end{table}

\section{50$\times$50 Scale-Up Validation}
\label{app:validation_50x50}

We repeat the preference validation on the full 100-agent population.
Each of the 50 male and 50 female evaluators scores all 50 opposite-gender candidates, yielding 2,500 cross-gender dyads and 5,000 directional evaluations.
Table~\ref{tab:validation_50x50} reports the exact alignment statistics, while Figure~\ref{fig:validation_50x50} additionally shows dimension-level consistency.
Kendall's $\tau$ and WPVR remain broadly comparable to the $10\times10$ results, indicating that ordinal alignment persists as the candidate pool expands; top-$K$ overlap is naturally more demanding in the larger pool.

\begin{table}[h]
\centering
\small
\caption{DeepSeek alignment with the empirical preference reference in the $50\times50$ validation market.}
\label{tab:validation_50x50}
\begin{tabular}{lcc}
\toprule
Metric & Male evaluators & Female evaluators \\
\midrule
Kendall's $\tau$ & 0.703 $\pm$ 0.054 & 0.607 $\pm$ 0.082 \\
WPVR & 0.054 $\pm$ 0.018 & 0.105 $\pm$ 0.039 \\
Overlap@3 & 0.460 $\pm$ 0.187 & 0.473 $\pm$ 0.232 \\
Overlap@5 & 0.684 $\pm$ 0.133 & 0.448 $\pm$ 0.186 \\
\bottomrule
\end{tabular}
\end{table}

\begin{figure}[h]
    \centering
    \includegraphics[width=0.95\linewidth]{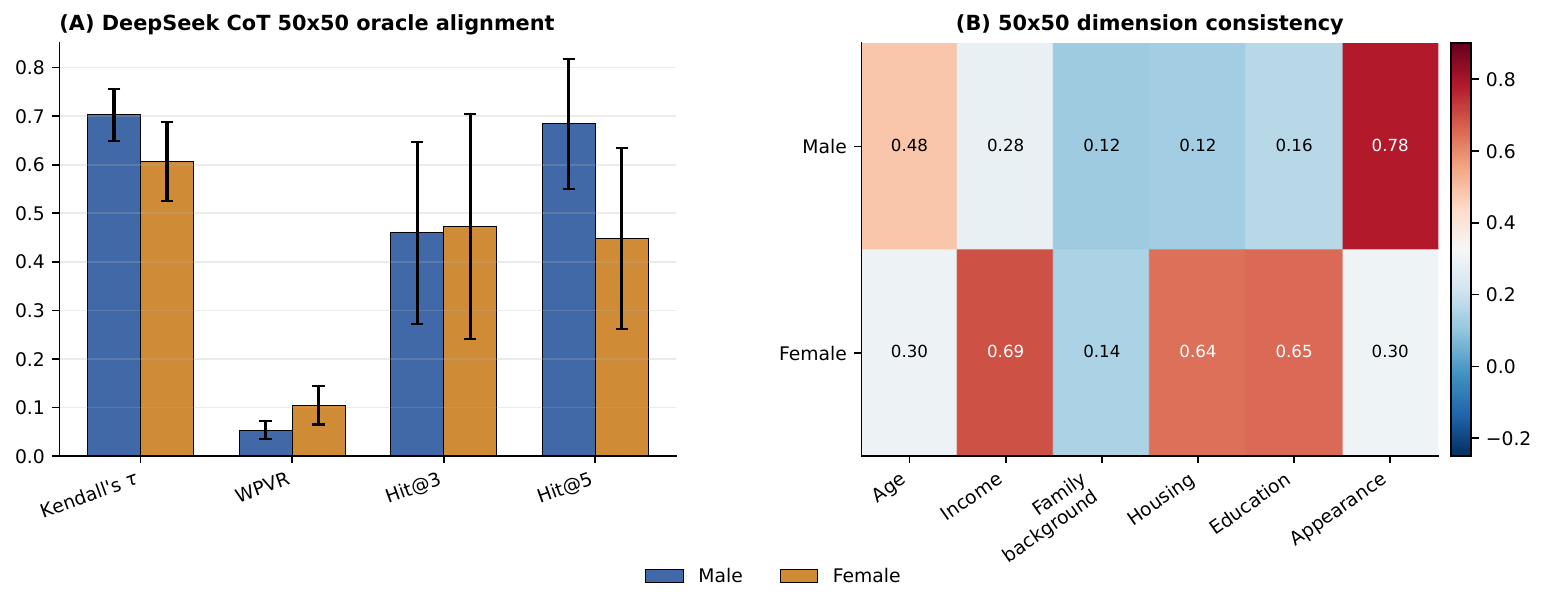}
    \caption{DeepSeek validation in the $50\times50$ market. Panel (a) reports reference-alignment metrics by evaluator gender; Panel (b) reports dimension-consistency correlations.}
    \label{fig:validation_50x50}
\end{figure}

\section{Context Construction for Local Reciprocal-Acceptance Learning}
\label{app:bandit_context}

This appendix specifies the six-dimensional \texttt{dyadic\_compact} representation used by the local Logistic-UCB models.
For proposer $i$ and candidate $j$, the context is written as $x^{(i)}_{j,t}=\phi(z_i,z_j)\in\mathbb{R}^6$, where $z_i$ and $z_j$ contain the six public persona attributes introduced in Section~\ref{sec:llm-agent-modeling}.
The representation summarizes only information observable before a proposal; it is intended to predict reciprocal acceptance, not to reconstruct the receiver's private utility function.

\begin{table}[h]
\centering
\small
\caption{Features in the local reciprocal-acceptance context. Gap and advantage variables follow the proposer-oriented encoding used in the implementation.}
\label{tab:bandit_context_features}
\begin{tabular}{ll}
\toprule
Feature & Interpretation \\
\midrule
Age gap & Age mismatch between proposer and candidate \\
Income gap & Relative income difference \\
Education gap & Relative education difference \\
Same background & Indicator for shared family background \\
Housing advantage & Relative housing-status advantage \\
Appearance gap & Relative appearance difference \\
\bottomrule
\end{tabular}
\end{table}

The decentralization constraint excludes receiver-side private quantities such as
$U_j(i)$, $U_j(\mu_t(j))$, and $b^R_{j,t}$, as well as the receiver's complete ranking, future accept/reject decision, other proposers' outcomes, and the parameters or histories of other local models.
Agent $i$ learns only from its own realized proposal history
$\mathcal{H}_{i,t}=\{(x_n^{(i)},y_n^{(i)})\}_{n=1}^{N_{i,t}}$.
A candidate who is observed but not approached therefore generates no label rather than a rejection.
This preserves the partial-feedback structure required by the contextual-bandit interpretation.

\section{Main Matching Experiment Settings and Metrics}
\label{app:matching_experiment_details}

\subsection{Experimental Protocol}

The main experiment uses the same 50 male and 50 female agents described in Section~\ref{sec:llm-agent-modeling}. The LLM-solver leverages DeepSeek-V4-Pro as the centralized solver, the same backbone of our LLM-ABM markets of 50$\times$50.
For every ordered cross-gender dyad, the validated LLM preference model produces a directional score $s_i(j)\in[0,100]$: Gale--Shapley, Axtell--Kimbrough and LLM-solver receive the preference information required by their mechanisms, whereas an LLM-agent gets its preference only after the local encounter.

The ``Info required'' column in Table~\ref{tab:main_matching_results} summarizes this distinction.
\emph{Whole} denotes complete market-wide preferences available before matching; \emph{Partial} denotes decentralized interaction that still begins from pre-specified preference lists; and \emph{Local} denotes the proposed encounter-based regime, in which no pairwise valuation or ranking is available to an agent before local exposure.

\begin{table}[h]
\centering
\small
\caption{Main LLM-ABM parameter configuration used in Table~\ref{tab:main_matching_results}.}
\label{tab:main_abm_parameters}
\begin{tabular}{ll}
\toprule
Parameter & Value \\
\midrule
Market size & 50 male agents and 50 female agents \\
Repeated runs & 50 random seeds \\
Main matching LLM scorer & DeepSeek V4 Pro, temperature $=0.5$, CoT prompt \\
Matching horizon & $T=500$ periods \\
Candidate set size & 25 candidates per activation when more are available \\
Reservation utility & $r=0.25$ on the normalized utility scale \\
Lock probability & $\rho=0.05$ per temporary relationship per period \\
Utility threshold & $\epsilon_U=0.01$ \\
Logistic-UCB coefficient & $\beta=1.0$ \\
Logistic regularization & $\lambda=1.0$ \\
Proposal cost and switching cost & 0 \\
\bottomrule
\end{tabular}
\end{table}

The three LLM-ABM policies share the same market dynamics and differ only in proposal selection.
\emph{Random eligible} samples uniformly from utility-improving candidates; \emph{Utility-only} chooses the largest immediate utility gain; and \emph{Bandit-UCB} maximizes $\Delta U_{ij,t}p_{ij,t}^{\mathrm{UCB},i}$.
Hence, Utility-only isolates the incremental role of reciprocal-acceptance learning, while Random eligible provides a non-strategic decentralized-search baseline.

\subsection{Evaluation Metrics}
Let
$\mathcal{P}(\mu)=\{(m,w):\mu(m)=w,\,m\in\mathcal{M},\,w\in\mathcal{W}\}$
denote the set of realized male--female pairs.

\paragraph{Mutual welfare.}
We report average bilateral desirability,
\begin{equation}
\mathrm{Mutual}(\mu)
=
\frac{1}{|\mathcal{P}(\mu)|}
\sum_{(m,w)\in\mathcal{P}(\mu)}
\frac{s_m(w)+s_w(m)}{2}.
\end{equation}
This rewards matches that are jointly valued by both partners rather than optimizing one side alone.

\paragraph{Realized rank and gender rank gap.}
Using score-based tied ranks,
\begin{equation}
\mathrm{rank}_i(j)
=
1+\sum_{k\in\mathcal{C}_{-i}}
\mathbf{1}\{s_i(k)>s_i(j)\},
\end{equation}
where $\mathcal{C}_{-i}$ is the full opposite-side candidate set used only for ex post evaluation.
If $\bar r_M(\mu)$ and $\bar r_W(\mu)$ are the average realized partner ranks of matched men and women, then
\begin{equation}
\mathrm{RankGap}(\mu)
=
\left|\bar r_M(\mu)-\bar r_W(\mu)\right|.
\end{equation}

\paragraph{Blocking pairs.}
Let $u_i(j)=s_i(j)/100$ and define the final benchmark
\begin{equation}
B_i(\mu)
=
\begin{cases}
\max\{r_i,u_i(\mu(i))\}, & \mu(i)\neq\varnothing,\\[2pt]
r_i, & \mu(i)=\varnothing.
\end{cases}
\end{equation}
An unmatched pair $(m,w)\notin\mathcal{P}(\mu)$ is blocking when
$u_m(w)>B_m(\mu)$ and $u_w(m)>B_w(\mu)$.
This is an ex post stability diagnostic; agents do not observe the blocking-pair set during matching.

\paragraph{Proposal volume.}
We count all proposal attempts as a measure of search intensity.
Because Gale--Shapley proposals are centralized algorithmic operations whereas LLM-ABM and Axtell--Kimbrough proposals arise from decentralized interaction, this quantity is slightly different for centralized and decentralized scenarios.

\section{Counterfactual Warm-Start Experiment}
\label{app:counterfactual_warm_start}

The counterfactual experiments in Section~\ref{sec:counterfactual_information_advantage} use the same $50\times50$ population and Bandit-UCB configuration as the main matching experiment.
Source runs use the same 50 seeds as the  main experiment, whereas counterfactual evaluation uses another 10 seeds.
The source phase provides learned proposer-side Logistic-UCB coefficients that are used only for initialization in the treatment conditions.

\paragraph{Individual warm start.}
Each male agent is treated as the focal agent in turn.
Only that agent receives the corresponding source-phase initialization; all other agents remain cold-start learners.
For every held-out seed, the warm-start run is paired with a cold-start Bandit-UCB run using the same seed, holding fixed the stochastic market realization associated with the evaluation seed.

\paragraph{Male-side warm start.}
We also initialize all male agents with the pooled male-proposer coefficients learned from the source runs, while female agents remain cold-start learners.
This stronger intervention tests whether information that can help an individual retains its value once it is shared across an entire side of the market.

Warm-starting changes only the initial coefficient vector $\theta$.
It does not alter LLM utilities, candidate exposure, reservation utilities, proposal rules, receiver behavior, or the observed proposal history, and subsequent feedback can revise or erase the initial belief.

\begin{table}[h]
\centering
\small
\caption{Distribution of individual warm-start effects across male focal agents.}
\label{tab:counterfactual_population_summary}
\begin{tabular}{lrrrrrr}
\toprule
Metric & Improved & Unchanged & Worsened & Avg. gain & Avg. loss & Mean effect \\
\midrule
Own rank & 11 & 26 & 13 & 0.95 & 0.46 & -0.15 \\
Own score & 11 & 26 & 13 & 0.95 & 0.55 & -0.22 \\
Proposals & 22 & 3 & 25 & 8.20 & 15.37 & -5.00 \\
\bottomrule
\end{tabular}
\end{table}

Table~\ref{tab:counterfactual_population_summary} uses the focal agent as the unit of analysis.
``Avg. gain'' reports the average signed improvement among agents with a positive effect, while ``Avg. loss'' reports the average absolute loss among agents with a non-positive effect, with unchanged agents contributing zero.
For rank and proposal volume, signs are oriented so that positive values indicate improvement; for score, higher values directly indicate improvement.

\section{Sensitivity Analysis of Exploration and Commitment}
\label{app:beta-rho-sensitivity}

\begin{figure}[ht]
    \centering
    \includegraphics[width=\linewidth]{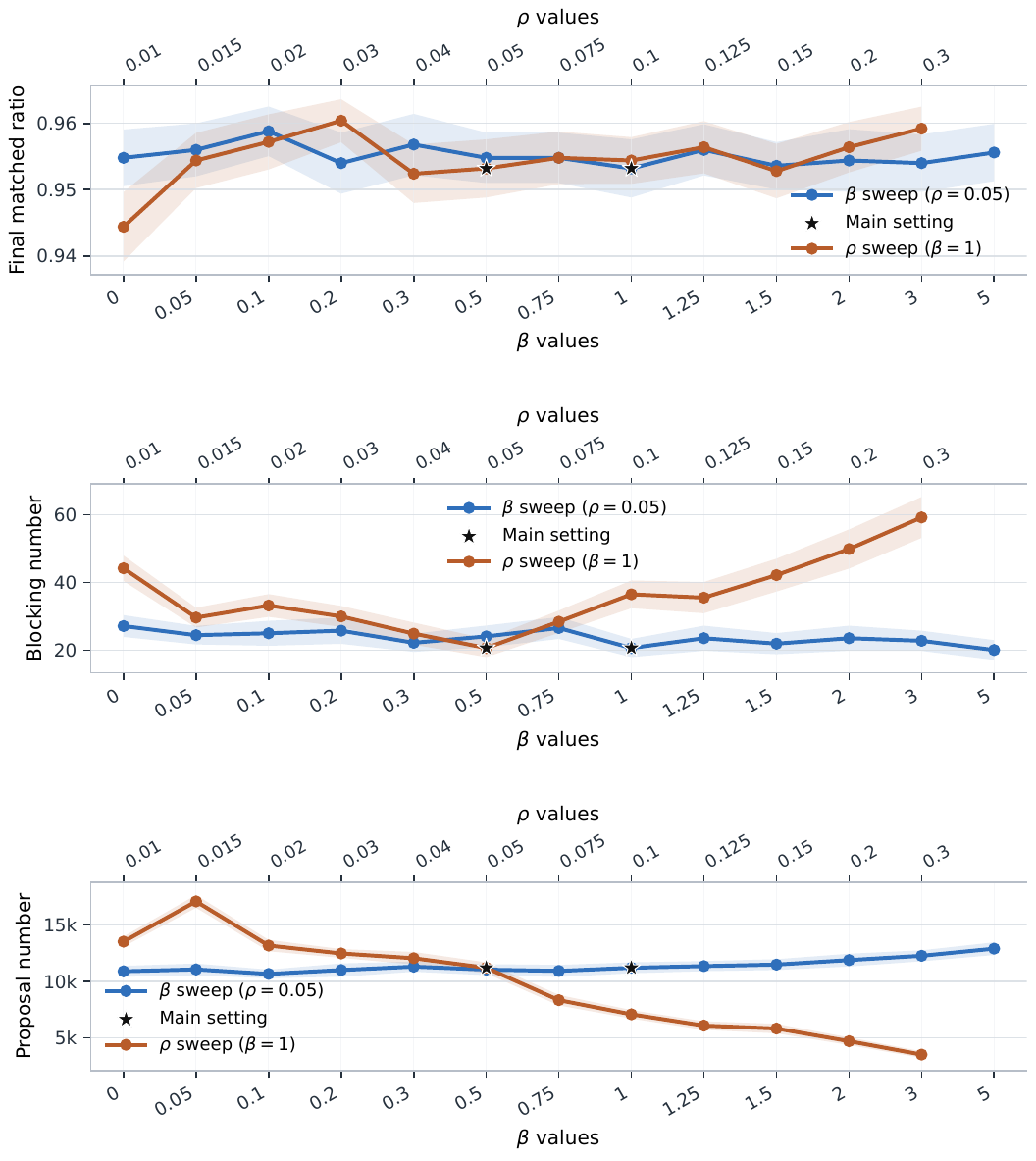}
    \caption{Sensitivity of market-level outcomes to exploration strength $\beta$ and lock probability $\rho$. The blue curve varies $\beta$ with $\rho=0.05$ fixed, while the orange curve varies $\rho$ with $\beta=1$ fixed. Shaded regions indicate 95\% confidence intervals across 50 seeds; stars mark the main experimental setting.}
    \label{fig:appendix-beta-rho-sensitivity}
\end{figure}

We further examine how the LLM-ABM market responds to two key behavioral parameters: the exploration coefficient $\beta$ in Bandit-UCB and the permanent lock probability $\rho$. The $\beta$ sweep varies the strength of uncertainty-driven exploration while holding $\rho=0.05$ fixed, whereas the $\rho$ sweep varies the rate at which temporary matches become permanent while holding $\beta=1$ fixed. All experiments use the same $50\times50$ market and 50 random seeds. For low-$\rho$ settings, we use a smaller per-activation proposal budget ($L$) to keep within-round search bounded.

Figure~\ref{fig:appendix-beta-rho-sensitivity} shows that the final matched ratio remains relatively stable across both parameter sweeps and stays close to the main setting for most values of $\beta$ and $\rho$. This indicates that the market reaches a high level of aggregate matching under a broad range of exploration and commitment intensities, although some variation in the share of unmatched agents remains.

The blocking-pair results are more sensitive to the commitment parameter. Varying $\beta$ changes the number of blocking pairs only moderately, suggesting that the main stability pattern is not driven by a narrow choice of exploration strength. In contrast, $\rho$ has a more pronounced effect. Very low commitment probabilities prolong the search process and leave more relationships unresolved, whereas very high commitment probabilities cause agents to exit the market earlier and can lock in less stable matches. The main setting, $\rho=0.05$, lies in a relatively low-blocking region of the sensitivity curve and provides a reasonable balance between continued search and timely commitment.

Proposal volume reveals a complementary search-intensity trade-off. Larger $\beta$ generally increases the number of proposals because stronger exploration encourages agents to consider more uncertain candidates. Increasing $\rho$, by contrast, sharply reduces proposal volume because agents become permanently matched and leave the active market more quickly. This reduction in search intensity, however, is accompanied by a higher blocking count at large $\rho$. Overall, the main configuration is not selected to optimize any single metric in isolation; rather, it provides a balanced operating point with a high final matched ratio, relatively few blocking pairs, and moderate proposal volume.

\section{Derivation of the Local Logistic-UCB Decision Rule}
\label{app:logistic_ucb_derivation}

This appendix derives the practical Fisher-information-based proposal rule used in Section~\ref{sec:local_bandit}. Appendix~\ref{app:theoretical_analysis} analyzes a calibrated variant and specifies its assumptions and scope.

\subsection{Local Logistic Acceptance Model}

For proposer $i$, suppose reciprocal acceptance is represented by a logistic working model
\begin{equation}
\Pr
\left(
Y_{i\rightarrow j,t}=1
\mid
x^{(i)}_{j,t}
\right)
=
\sigma
\left(
{x^{(i)}_{j,t}}^\top\theta_i
\right),
\label{eq:appendix_logistic_model}
\end{equation}
where
\begin{equation}
\sigma(z)
=
\frac{1}{1+\exp(-z)}.
\label{eq:appendix_sigmoid}
\end{equation}

Given local proposal history $\mathcal{H}_{i,t}$, the regularized negative log-likelihood is
\begin{equation}
\mathcal{L}_{i,t}(\theta)
=
\sum_{n=1}^{N_{i,t}}
\left[
\log
\left(
1+\exp
\left(
{x_n^{(i)}}^\top\theta
\right)
\right)
-
y_n^{(i)}
{x_n^{(i)}}^\top\theta
\right]
+
\frac{\lambda_i}{2}
\|\theta\|_2^2.
\label{eq:appendix_logistic_loss}
\end{equation}

Its gradient is
\begin{equation}
\nabla
\mathcal{L}_{i,t}(\theta)
=
\sum_{n=1}^{N_{i,t}}
\left[
\sigma
\left(
{x_n^{(i)}}^\top\theta
\right)
-
y_n^{(i)}
\right]
x_n^{(i)}
+
\lambda_i\theta.
\label{eq:appendix_logistic_gradient}
\end{equation}

The Hessian is
\begin{equation}
\nabla^2
\mathcal{L}_{i,t}(\theta)
=
\lambda_i I
+
\sum_{n=1}^{N_{i,t}}
\sigma
\left(
{x_n^{(i)}}^\top\theta
\right)
\left[
1-
\sigma
\left(
{x_n^{(i)}}^\top\theta
\right)
\right]
x_n^{(i)}
{x_n^{(i)}}^\top.
\label{eq:appendix_logistic_hessian}
\end{equation}

Evaluating Equation~\ref{eq:appendix_logistic_hessian} at the fitted parameter $\widehat\theta_{i,t}$ yields
\begin{equation}
V_{i,t}
=
\lambda_i I
+
\sum_{n=1}^{N_{i,t}}
\widehat p^{(i)}_{n,t}
\left(
1-\widehat p^{(i)}_{n,t}
\right)
x_n^{(i)}
{x_n^{(i)}}^\top,
\label{eq:appendix_fisher_matrix}
\end{equation}
which is the local curvature matrix used in the main text.

\subsection{Uncertainty in the Logit Predictor}

Suppose, for interpretation, that a local confidence region around the fitted parameter can be approximated by
\begin{equation}
\mathcal{C}_{i,t}
=
\left\{
\theta:
\left\|
\theta-\widehat\theta_{i,t}
\right\|_{V_{i,t}}
\leq
\beta_{i,t}
\right\},
\label{eq:appendix_confidence_ellipsoid}
\end{equation}
where
\begin{equation}
\|z\|_V
=
\sqrt{z^\top Vz}.
\end{equation}

For a candidate context $x^{(i)}_{j,t}$,
\begin{align}
{x^{(i)}_{j,t}}^\top
\left(
\theta-\widehat\theta_{i,t}
\right)
&=
\left(
V_{i,t}^{-1/2}
x^{(i)}_{j,t}
\right)^\top
\left(
V_{i,t}^{1/2}
\left(
\theta-\widehat\theta_{i,t}
\right)
\right)
\\
&\leq
\left\|
V_{i,t}^{-1/2}
x^{(i)}_{j,t}
\right\|_2
\left\|
V_{i,t}^{1/2}
\left(
\theta-\widehat\theta_{i,t}
\right)
\right\|_2
\\
&\leq
\beta_{i,t}
\sqrt{
{x^{(i)}_{j,t}}^\top
V_{i,t}^{-1}
x^{(i)}_{j,t}
}.
\label{eq:appendix_ucb_projection}
\end{align}

This motivates the uncertainty measure
\begin{equation}
s^{(i)}_{j,t}
=
\sqrt{
{x^{(i)}_{j,t}}^\top
V_{i,t}^{-1}
x^{(i)}_{j,t}
}.
\label{eq:appendix_context_uncertainty}
\end{equation}

The corresponding optimistic linear predictor is
\begin{equation}
z_{ij,t}^{\mathrm{UCB},i}
=
\widehat\theta_{i,t}^{\top}
x^{(i)}_{j,t}
+
\beta_{i,t}
s^{(i)}_{j,t}.
\label{eq:appendix_logit_ucb}
\end{equation}

Since the sigmoid link is strictly increasing,
\begin{equation}
z_1\leq z_2
\quad\Longrightarrow\quad
\sigma(z_1)\leq\sigma(z_2),
\end{equation}
we obtain the optimistic acceptance prediction
\begin{equation}
p_{ij,t}^{\mathrm{UCB},i}
=
\sigma
\left(
\widehat\theta_{i,t}^{\top}
x^{(i)}_{j,t}
+
\beta_{i,t}
s^{(i)}_{j,t}
\right).
\label{eq:appendix_probability_ucb}
\end{equation}

This construction follows the general optimistic principle used in generalized-linear contextual bandits \citep{filippi2010parametric,li2017provably}.
Our particular use of the fitted logistic Hessian in Equation~\ref{eq:appendix_fisher_matrix} should be understood as a local Fisher-information or Wald/Laplace-style approximation.
Because the Hessian of a logistic model depends on the unknown parameter itself, more refined Logistic-UCB analyses can require nonlinear confidence sets and stronger regularity conditions.
Accordingly, in the main experiments, $\beta_{i,t}$ is treated as an exploration coefficient rather than as an exact finite-sample confidence radius, and $\widehat\theta_{i,t}$ is obtained from the unconstrained regularized problem. The analyzable variant in Appendix~\ref{app:theoretical_analysis} instead restricts the estimator to a norm ball, which is what makes its curvature constant follow from primitive bounds.

\subsection{Proposal Index as Optimistic Expected Utility Improvement}

The bandit observes the binary response
\begin{equation}
Y_{i\rightarrow j,t}
\in\{0,1\},
\end{equation}
but the economic objective of the proposer is not simply to maximize acceptance probability.

For an eligible candidate $j\in\mathcal{E}_t(i)$, define the realized utility improvement from a proposal as
\begin{equation}
R_{i\rightarrow j,t}
=
Y_{i\rightarrow j,t}
\Delta U_{ij,t}.
\label{eq:appendix_realized_reward}
\end{equation}

Conditional on the information available to proposer $i$,
\begin{equation}
\mathbb{E}
\left[
R_{i\rightarrow j,t}
\mid
x^{(i)}_{j,t}
\right]
=
\Delta U_{ij,t}
\Pr
\left(
Y_{i\rightarrow j,t}=1
\mid
x^{(i)}_{j,t}
\right).
\label{eq:appendix_expected_reward}
\end{equation}

The eligible-set restriction implies
\begin{equation}
\Delta U_{ij,t}>0.
\label{eq:appendix_positive_delta_u}
\end{equation}
Therefore, replacing the unknown reciprocal-acceptance probability with an optimistic estimate preserves the direction of the upper bound and leads to
\begin{equation}
I_{ij,t}
=
\Delta U_{ij,t}
p_{ij,t}^{\mathrm{UCB},i}.
\label{eq:appendix_proposal_index}
\end{equation}

This derivation clarifies why the positivity restriction in Equation~\ref{eq:eligible_candidate_set} is important.
If $\Delta U_{ij,t}<0$, multiplication by a probability upper bound would reverse the relevant inequality, and Equation~\ref{eq:appendix_proposal_index} would no longer admit the same optimistic expected-improvement interpretation.

\section{Theoretical Analysis of Local Learning and Bilateral Rematching}
\label{app:theoretical_analysis}

\subsection{Overview of Theoretical Results}
\label{app:theory_summary}

The theoretical analysis fixes one proposer and studies its own sequence of realized proposal attempts. It is intentionally local: the benchmark is a myopic oracle facing the same currently eligible candidates, rather than a market-wide matching trajectory. Under this scope, the appendix establishes five results.
\begin{enumerate}
    \item[\textbf{R1.}] \textbf{Oracle one-step optimality.} If current reciprocal-acceptance probabilities were known, maximizing $\Delta U\times p$ maximizes the proposer's conditional expected one-step utility improvement.
    \item[\textbf{R2.}] \textbf{Local pseudo-regret under a correctly specified working model.} Under predictable bounded contexts, bounded positive utility improvements, and a proposer-specific \emph{realizable} logistic acceptance model, a theoretically calibrated Logistic-UCB rule with a norm-constrained local estimator has sublinear local pseudo-regret with high probability. Exact realizability need not hold in the mechanism of Section~\ref{sec:local_bandit}; R2 is stated as a benchmark that isolates the statistical cost of learning reciprocation and supplies the reference rate for R3.
    \item[\textbf{R3.}] \textbf{Dynamic misspecification.} The decentralized design of Section~\ref{sec:local_bandit} induces a non-stationary acceptance law, so the logistic model is a working approximation rather than a data-generating process. If it matches the true, time-varying acceptance probability up to a bounded envelope $\eta_n$, the regret bound separates statistical learning error from an explicit cumulative misspecification term. Observable dynamic context can shrink this term but cannot generally eliminate receiver-private state information by construction.
    \item[\textbf{R4.}] \textbf{Strict bilateral improvement.} Every accepted rematching strictly improves the partner utility of both consenting agents relative to their pre-proposal benchmarks.
    \item[\textbf{R5.}] \textbf{Matched-pair welfare non-monotonicity.} Strict bilateral improvement does not imply monotone matched-pair welfare because an accepted rematching can displace previous partners. Consequently, local no-regret learning alone does not imply terminal optimality of the matched-pair welfare metric or stability.
\end{enumerate}
The analysis uses a norm-constrained estimator and calibrated confidence radii, whereas the experiments use an unconstrained estimator with fixed exploration. Remark~\ref{rem:theory_realizability_scope} explains the relationship between the acceptance model and the simulated mechanism.

\subsection{Formal Results}
\label{app:theory_results}

\paragraph{Proposal-time notation and learner filtration.}
Fix proposer $i$ and index its own realized proposal attempts by $n=1,2,\ldots$, distinct from the global market-period index. Let $\mathcal F_{n-1}$ be the information available to the proposer immediately before attempt $n$. It contains the proposer's past proposal contexts and outcomes, the currently exposed and eligible candidate set, all public information available for current candidates, the current fitted local model, and algorithmic randomness revealed before the current proposal. It excludes receiver-private quantities that the mechanism does not reveal, including the receiver's private valuation of the proposer and the value of the receiver's incumbent relationship.

At attempt $n$, let $\mathcal R_n$ be the current eligible candidate set. For each $j\in\mathcal R_n$, the proposer observes $x_{j,n}\in\mathbb R^d$ and a positive utility improvement $\Delta U_n(j)>0$. Let $Y_n(j)\in\{0,1\}$ denote the potential reciprocal-acceptance response that would be observed if candidate $j$ were selected at attempt $n$. If candidate $j$ is selected, write
\[
J_n=j,\qquad X_n:=x_{J_n,n},\qquad Y_n:=Y_n(J_n).
\]
For each current candidate define
\[
p_n(j):=\mathbb P\!\left(Y_n(j)=1\mid \mathcal F_{n-1}\right),
\]
and the potential one-step proposal reward
\[
R_n(j):=Y_n(j)\Delta U_n(j),
\qquad
g_n(j):=\mathbb E[R_n(j)\mid\mathcal F_{n-1}]
=\Delta U_n(j)p_n(j).
\]
The selected context $X_n$ is required to be $\mathcal F_{n-1}$-measurable. Thus the policy may adapt arbitrarily to past outcomes and to the current candidate set, but must select the current context before observing $Y_n$.

\begin{proposition}[R1: Oracle one-step proposal optimality]
\label{prop:theory_one_step}
Condition on $\mathcal F_{n-1}$ and suppose the proposer makes one current proposal from $\mathcal R_n$. If the true conditional acceptance probabilities $\{p_n(j):j\in\mathcal R_n\}$ were known, then every
\[
J_n^\star\in\arg\max_{j\in\mathcal R_n}\Delta U_n(j)p_n(j)
\]
maximizes the proposer's conditional expected one-step utility improvement.
\end{proposition}

This is a myopic statement. If a rejection leaves a continuation opportunity within the same activation, optimizing an entire sequence of proposals is a dynamic decision problem and need not coincide with ranking candidates solely by one-step expected improvement.

\paragraph{Assumptions for local learning.}
\begin{assumption}[Predictable bounded contexts]
\label{ass:theory_contexts}
For every $n$, $X_n$ is $\mathcal F_{n-1}$-measurable and $\|X_n\|_2\le L$.
\end{assumption}

\begin{assumption}[Realizable reciprocal acceptance]
\label{ass:theory_realizable}
There exists a fixed proposer-specific parameter $\theta^\star\in\mathbb R^d$ with $\|\theta^\star\|_2\le S$ such that, for every $n$ and \emph{every} currently eligible candidate $j\in\mathcal R_n$,
\[
p_n(j)
=\mathbb P\!\left(Y_n(j)=1\mid\mathcal F_{n-1}\right)
=\sigma(x_{j,n}^\top\theta^\star),
\qquad
\sigma(z)=\frac{1}{1+e^{-z}}.
\]
In particular, because $J_n$ is $\mathcal F_{n-1}$-measurable, $\mathbb E[Y_n\mid\mathcal F_{n-1}]=\sigma(X_n^\top\theta^\star)$.
\end{assumption}
The statement is made over the potential outcomes $\{Y_n(j)\}_{j\in\mathcal R_n}$ rather than only over the selected action, because the regret analysis compares the chosen candidate with an unselected oracle candidate and therefore requires the model to describe counterfactual arms as well.

\begin{remark}[Assumption~\ref{ass:theory_realizable} need not hold exactly in our mechanism]
\label{rem:theory_realizability_scope}
Assumption~\ref{ass:theory_realizable} need not be satisfied by the matching system of Section~\ref{sec:local_bandit}. Acceptance there is generated by the threshold rule $Y_{i\to j,t}=\mathds{1}\{U_j(i)>b^R_{j,t}\}$, whereas the local context $x^{(i)}_{j,t}=\phi(z_i,z_j)$ of Appendix~\ref{app:bandit_context} is built from static public persona attributes and is therefore constant across attempts for a fixed ordered pair $(i,j)$. Because the receiver's benchmark $b^R_{j,t}$ moves with its incumbent relationship and is excluded from the proposer's information set, two histories can share the same proposer-observable dyadic context while inducing different conditional acceptance probabilities. Hence a single fixed $\theta^\star$ need not represent $p_n(j)$ across market states. This is a consequence of the information design rather than a modeling oversight: receiver-private and time-varying thresholds are deliberately withheld from proposers.

We nonetheless state Theorem~\ref{thm:theory_realizable} under Assumption~\ref{ass:theory_realizable} as a correctly specified benchmark. It isolates the purely statistical difficulty of learning reciprocation when the working model is correct, and it supplies the reference rate against which the misspecification penalty is measured. Theorem~\ref{thm:theory_misspec}, which replaces realizability by the bounded-envelope Assumption~\ref{ass:theory_misspec}, is formulated for the form of dynamic misspecification induced by the simulated mechanism; its sufficient conditions are not claimed to be verified by the simulator.
\end{remark}

\begin{assumption}[Bounded positive utility improvement]
\label{ass:theory_utility}
For every current candidate, $0<\Delta U_n(j)\le D$ for a finite constant $D$. Under utilities normalized to $[0,1]$, one may take $D\le1$.
\end{assumption}

\paragraph{Norm-constrained local estimator.}
Let
\[
\Theta:=\left\{\theta\in\mathbb R^d:\|\theta\|_2\le S\right\}
\]
be the parameter set implied by the norm bound on $\theta^\star$, and define the regularized logistic objective together with its constrained minimizer
\[
\mathcal L_n(\theta)
=
\sum_{m=1}^{n-1}
\left[
\log(1+e^{X_m^\top\theta})-Y_mX_m^\top\theta
\right]
+\frac{\lambda}{2}\|\theta\|_2^2,
\qquad
\widehat\theta_n:=\arg\min_{\theta\in\Theta}\mathcal L_n(\theta).
\]
Because $\mathcal L_n$ is strictly convex for $\lambda>0$ and $\Theta$ is convex and compact, $\widehat\theta_n$ exists and is unique. Restricting the estimator to $\Theta$ is what turns the logistic curvature below into a derived quantity determined by the primitive constants $L$ and $S$, rather than a condition imposed on the realized estimator sequence. Define the matrices
\[
G_n:=\lambda I+\sum_{m=1}^{n-1}X_mX_m^\top,
\qquad
V_n:=\lambda I+\sum_{m=1}^{n-1}\sigma'(X_m^\top\widehat\theta_n)X_mX_m^\top,
\]
and the practical Fisher width
\[
s_n(x):=\sqrt{x^\top V_n^{-1}x}.
\]

\begin{lemma}[Uniform local logistic curvature]
\label{lem:theory_curvature}
Suppose Assumption~\ref{ass:theory_contexts} holds and $\theta^\star\in\Theta$, and set
\[
\underline\kappa
:=
\sigma'(LS)
=
\frac{e^{-LS}}{\left(1+e^{-LS}\right)^2}
\in(0,\tfrac14].
\]
Then for every $n$, every $m<n$, and every $u\in[0,1]$,
\[
\sigma'\!\left(
X_m^\top[\theta^\star+u(\widehat\theta_n-\theta^\star)]
\right)
\ge \underline\kappa.
\]
\end{lemma}
Lemma~\ref{lem:theory_curvature} supplies the curvature needed to convert score concentration into parameter concentration. The mechanism is that $\theta^\star$ and $\widehat\theta_n$ both lie in the convex set $\Theta$, so the entire segment joining them lies in $\Theta$ and the logistic predictions entering the mean-value expansion cannot be arbitrarily saturated. Under an unconstrained estimator the analogous statement would have to be assumed directly on the fitted path, because $\|\widehat\theta_n\|_2$ is then only controlled by a bound that grows with the sample size.
For $\delta\in(0,1)$, define
\begin{equation}
\label{eq:theory_beta_det}
\beta_n(\delta)
:=
\frac{1}{\underline\kappa}
\left[
\sqrt\lambda S
+\frac12
\sqrt{
\log\frac{\det G_n}{\lambda^d}
+2\log\frac1\delta
}
\right],
\end{equation}
and the deterministic upper bound
\begin{equation}
\label{eq:theory_beta_closed}
\bar\beta_n(\delta)
:=
\frac{1}{\underline\kappa}
\left[
\sqrt\lambda S
+\frac12
\sqrt{
d\log\!\left(1+\frac{(n-1)L^2}{\lambda d}\right)
+2\log\frac1\delta
}
\right].
\end{equation}
The theoretically calibrated optimistic probability is
\[
p_n^{\mathrm{UCB}}(j)
:=
\sigma\!\left(
x_{j,n}^\top\widehat\theta_n
+\beta_n(\delta)s_n(x_{j,n})
\right).
\]

\begin{definition}[Local proposal pseudo-regret]
\label{def:theory_regret}
Let
\[
J_n^\star\in\arg\max_{j\in\mathcal R_n}\Delta U_n(j)p_n(j),
\qquad
J_n\in\arg\max_{j\in\mathcal R_n}\Delta U_n(j)p_n^{\mathrm{UCB}}(j).
\]
The instantaneous local pseudo-regret and cumulative pseudo-regret are
\[
r_n
:=
\Delta U_n(J_n^\star)p_n(J_n^\star)
-\Delta U_n(J_n)p_n(J_n),
\qquad
\mathcal R_N:=\sum_{n=1}^N r_n.
\]
\end{definition}

\begin{theorem}[R2: Local regret under realizable reciprocal acceptance]
\label{thm:theory_realizable}
Suppose Assumptions~\ref{ass:theory_contexts}--\ref{ass:theory_utility} hold, $\lambda>0$, and $\underline\kappa=\sigma'(LS)$ is the curvature constant of Lemma~\ref{lem:theory_curvature}. Then, with probability at least $1-\delta$, simultaneously for all proposal attempts $n$,
\begin{equation}
\label{eq:theory_param_confidence}
\|\widehat\theta_n-\theta^\star\|_{G_n}
\le
\beta_n(\delta)
\le
\bar\beta_n(\delta),
\end{equation}
and, for every current candidate context $x$,
\begin{equation}
\label{eq:theory_valid_optimism}
\sigma(x^\top\theta^\star)
\le
\sigma\!\left(
x^\top\widehat\theta_n
+\beta_n(\delta)\sqrt{x^\top V_n^{-1}x}
\right).
\end{equation}
Consequently,
\begin{equation}
\label{eq:theory_regret_realizable}
\mathcal R_N
\le
\frac{D\,\bar\beta_N(\delta)}{2}
\sqrt{
\frac{Nd}{\underline\kappa}
\left(1+\frac{L^2}{\lambda}\right)
\log\!\left(1+\frac{NL^2}{\lambda d}\right)
}.
\end{equation}
For fixed $L,S,\lambda,D,$ and $\underline\kappa$,
\[
\mathcal R_N
=
\widetilde{\mathcal O}\!\left(\frac{d\sqrt N}{\underline\kappa^{3/2}}\right),
\qquad
\frac{\mathcal R_N}{N}\to0.
\]
\end{theorem}
The benchmark in Theorem~\ref{thm:theory_realizable} is a myopic oracle that faces the same current candidate set and knows the true current reciprocal-acceptance probabilities. The result isolates the local learning loss caused by uncertainty about reciprocation; it does not compare full-market trajectories generated by different policies. As stated in Remark~\ref{rem:theory_realizability_scope}, its realizability hypothesis need not be met by our mechanism, so Theorem~\ref{thm:theory_realizable} should be read as the correctly specified reference point for the misspecified analysis that follows.

\paragraph{Dynamic misspecification and observable context.}
The matching system need not obey a stationary logistic law because a receiver's acceptance threshold changes when its current relationship changes, while the proposer is deliberately denied receiver-private information. We therefore also analyze the logistic model as a working approximation based only on proposer-observable information.

\begin{assumption}[Bounded dynamic logistic misspecification]
\label{ass:theory_misspec}
There exists a fixed reference parameter $\theta^\star$ with $\|\theta^\star\|_2\le S$, i.e.\ $\theta^\star\in\Theta$, and a predictable nonnegative sequence $\{\eta_n\}$ such that, for every $j\in\mathcal R_n$,
\[
p_n(j)
=
\sigma(x_{j,n}^\top\theta^\star)+b_n(j),
\qquad
|b_n(j)|\le\eta_n,
\]
with $p_n(j)\in[0,1]$.
\end{assumption}
One may augment the public dyadic context with dynamic variables known before the current proposal, such as market stage, the proposer's accumulated proposal count, or the proposer's own recent acceptance statistics. Such features remain predictable and can reduce misspecification. They cannot, however, generally guarantee exact realizability under the present information design: two market states may induce the same proposer-observable context but different receiver-private incumbent values and therefore different acceptance thresholds. Hence richer observable context should be interpreted as potentially shrinking the envelope $\eta_n$, rather than automatically eliminating it; increasing context dimension also enters the statistical bound through $d$.

For selected historical actions write $b_m:=b_m(J_m)$. Define
\begin{equation}
\label{eq:theory_beta_mis}
\beta_n^{\mathrm{mis}}(\delta)
:=
\frac{1}{\underline\kappa}
\left[
\sqrt\lambda S
+\frac12
\sqrt{
\log\frac{\det G_n}{\lambda^d}
+2\log\frac1\delta
}
+\sqrt{\sum_{m=1}^{n-1}\eta_m^2}
\right],
\end{equation}
and
\begin{equation}
\label{eq:theory_beta_mis_closed}
\bar\beta_n^{\mathrm{mis}}(\delta)
:=
\frac{1}{\underline\kappa}
\left[
\sqrt\lambda S
+\frac12
\sqrt{
d\log\!\left(1+\frac{(n-1)L^2}{\lambda d}\right)
+2\log\frac1\delta
}
+\sqrt{\sum_{m=1}^{n-1}\eta_m^2}
\right].
\end{equation}
Let
\[
\bar p_n^{\mathrm{UCB}}(j)
:=
\sigma\!\left(
x_{j,n}^\top\widehat\theta_n
+\beta_n^{\mathrm{mis}}(\delta)s_n(x_{j,n})
\right),
\]
and suppose the policy selects
\[
J_n\in\arg\max_{j\in\mathcal R_n}\Delta U_n(j)\bar p_n^{\mathrm{UCB}}(j).
\]

\begin{theorem}[R3: Local regret under bounded dynamic misspecification]
\label{thm:theory_misspec}
Suppose Assumptions~\ref{ass:theory_contexts}, \ref{ass:theory_utility}, and~\ref{ass:theory_misspec} hold, with $\underline\kappa=\sigma'(LS)$ as in Lemma~\ref{lem:theory_curvature}. Assumption~\ref{ass:theory_realizable} is \emph{not} required. Then, with probability at least $1-\delta$, simultaneously for all $n$,
\[
\|\widehat\theta_n-\theta^\star\|_{G_n}
\le
\beta_n^{\mathrm{mis}}(\delta).
\]
Moreover, local pseudo-regret relative to the oracle that knows the true current probabilities $p_n(j)$ satisfies
\begin{equation}
\label{eq:theory_regret_mis}
\mathcal R_N
\le
\frac{D\,\bar\beta_N^{\mathrm{mis}}(\delta)}{2}
\sqrt{
\frac{Nd}{\underline\kappa}
\left(1+\frac{L^2}{\lambda}\right)
\log\!\left(1+\frac{NL^2}{\lambda d}\right)
}
+2D\sum_{n=1}^N\eta_n.
\end{equation}
A simple sufficient condition for vanishing average local pseudo-regret is
\[
\sum_{n=1}^N\eta_n=o\!\left(\frac{N}{\log N}\right),
\]
assuming $0\le\eta_n\le1$ and the remaining problem constants are fixed.
\end{theorem}
The misspecification term in Theorem~\ref{thm:theory_misspec} is not only a feature-engineering residual: under the decentralized information restriction it can also represent state variables intentionally hidden from the proposer. Theorem~\ref{thm:theory_misspec} is therefore formulated to accommodate the form of dynamic misspecification induced by the mechanism of Section~\ref{sec:local_bandit}, and it reduces to Theorem~\ref{thm:theory_realizable} when $\eta_n\equiv0$. Verifying the sufficient condition $\sum_n\eta_n=o(N/\log N)$ requires knowledge of the receiver-side state that the design withholds from proposers; characterizing $\eta_n$ endogenously is left to future work.

\paragraph{Bilateral rematching.}
For any agent $k$, define its pre-rematching individual benchmark utility
\[
v_k(\mu_t)
:=
\begin{cases}
r_k, & \mu_t(k)=\varnothing,\\
U_k(\mu_t(k)), & \mu_t(k)\neq\varnothing.
\end{cases}
\]

\begin{proposition}[R4: Strict bilateral improvement of every accepted rematching]
\label{prop:theory_bilateral}
Suppose proposer $i$ selects an eligible candidate $j$ at time $t$ and the proposal is accepted. Then the newly formed pair $(i,j)$ strictly improves the individual partner utility of both consenting agents relative to their pre-proposal benchmarks:
\[
U_i(j)>v_i(\mu_t),
\qquad
U_j(i)>v_j(\mu_t).
\]
More precisely,
\[
U_i(j)-b_{i,t}>\epsilon_U\ge0,
\qquad
U_j(i)-b^R_{j,t}>0.
\]
\end{proposition}

Let $\mathcal P(\mu)$ denote the set of realized cross-side pairs in matching $\mu$. For consistency with the empirical mutual-welfare metric, define total matched-pair welfare
\[
W_{\mathrm{tot}}(\mu)
:=
\sum_{(a,b)\in\mathcal P(\mu)}
\frac{U_a(b)+U_b(a)}{2},
\]
and, when $\mathcal P(\mu)\neq\varnothing$, average mutual welfare
\[
W_{\mathrm{avg}}(\mu)
:=
\frac{W_{\mathrm{tot}}(\mu)}{|\mathcal P(\mu)|}.
\]

\begin{proposition}[R5: Matched-pair welfare need not be monotone]
\label{prop:theory_nonmono}
There exists a finite two-sided market and a temporary matching $\mu_t$ such that an accepted proposal is a strict bilateral improvement for the proposer and receiver as in Proposition~\ref{prop:theory_bilateral}, yet the immediate rematching strictly decreases both total matched-pair welfare and average mutual welfare:
\[
W_{\mathrm{tot}}(\mu_{t+1})<W_{\mathrm{tot}}(\mu_t),
\qquad
W_{\mathrm{avg}}(\mu_{t+1})<W_{\mathrm{avg}}(\mu_t).
\]
\end{proposition}
The reason is a matching externality: the bilateral acceptance test internalizes the preferences of the two consenting agents, but not the losses of partners displaced by the rematching.

\begin{corollary}[Local learning guarantees do not imply global matching optimality]
\label{cor:theory_local_global}
Even if proposer-side local pseudo-regret in Theorem~\ref{thm:theory_realizable} or Theorem~\ref{thm:theory_misspec} is sublinear, this alone does not imply monotone matched-pair welfare, terminal optimality of that welfare metric, or convergence to a stable matching.
\end{corollary}

\subsection{Proofs}
\label{app:theory_proofs}

\paragraph{Proof of Proposition~\ref{prop:theory_one_step}.}
\begin{proof}
Condition on $\mathcal F_{n-1}$. Since $\Delta U_n(j)$ is known before the current response, for every current candidate $j$,
\[
\begin{aligned}
\mathbb E[R_n(j)\mid\mathcal F_{n-1}]
&=
\mathbb E[Y_n(j)\Delta U_n(j)\mid\mathcal F_{n-1}]\\
&=
\Delta U_n(j)\mathbb E[Y_n(j)\mid\mathcal F_{n-1}]\\
&=
\Delta U_n(j)p_n(j).
\end{aligned}
\]
Therefore any maximizer of $\Delta U_n(j)p_n(j)$ maximizes conditional expected one-step utility improvement.
\end{proof}

\paragraph{Specialization of Abbasi-Yadkori--P\'al--Szepesv\'ari.}
The only external concentration result used in the learning proofs is the time-uniform self-normalized inequality of \citet{abbasi2011improved}, Theorem~1. Their theorem applies to a predictable vector sequence $X_m$ and a conditionally mean-zero, $R$-sub-Gaussian scalar noise sequence $\varepsilon_m$, with
\[
S_t=\sum_{m=1}^t\varepsilon_mX_m,
\qquad
\bar V_t=V+\sum_{m=1}^tX_mX_m^\top,
\]
and yields, with probability at least $1-\delta$, simultaneously for all $t$,
\[
\|S_t\|_{\bar V_t^{-1}}^2
\le
2R^2\log\!\left(
\frac{\det(\bar V_t)^{1/2}}{\det(V)^{1/2}\delta}
\right).
\]
Under Assumption~\ref{ass:theory_realizable}, set
\[
\varepsilon_m:=Y_m-\sigma(X_m^\top\theta^\star),
\qquad
V=\lambda I,
\qquad
t=n-1.
\]
Then $\mathbb E[\varepsilon_m\mid\mathcal F_{m-1}]=0$. Conditional on $\mathcal F_{m-1}$, $Y_m\in\{0,1\}$, so Hoeffding's lemma implies that the centered Bernoulli noise $\varepsilon_m$ is conditionally $1/2$-sub-Gaussian. Moreover $\bar V_{n-1}=G_n$. Therefore, with probability at least $1-\delta$, simultaneously for every $n$,
\begin{equation}
\label{eq:theory_selfnorm}
\left\|
\sum_{m=1}^{n-1}\varepsilon_mX_m
\right\|_{G_n^{-1}}
\le
\frac12
\sqrt{
\log\frac{\det G_n}{\lambda^d}
+2\log\frac1\delta
}.
\end{equation}
The adaptive proposal policy is compatible with this result because the theorem requires predictability rather than i.i.d. contexts: $X_n$ may depend on the entire observed history and current candidate set as long as it is chosen before $Y_n$ is observed.

\paragraph{Proof of Lemma~\ref{lem:theory_curvature}.}
\begin{proof}
Fix $n$, $m<n$, and $u\in[0,1]$, and set $\theta_u:=\theta^\star+u(\widehat\theta_n-\theta^\star)$. Both $\theta^\star$ and $\widehat\theta_n$ belong to $\Theta$, which is convex, so $\theta_u\in\Theta$ and $\|\theta_u\|_2\le S$. By Cauchy--Schwarz and Assumption~\ref{ass:theory_contexts},
\[
\left|X_m^\top\theta_u\right|
\le\|X_m\|_2\|\theta_u\|_2
\le LS.
\]
The derivative $\sigma'(z)=\sigma(z)(1-\sigma(z))$ is even and strictly decreasing in $|z|$, so $|z|\le LS$ implies $\sigma'(z)\ge\sigma'(LS)=\underline\kappa$. Finally $\underline\kappa\le\sigma'(0)=1/4$, and $\underline\kappa>0$ because $L$ and $S$ are finite.
\end{proof}

\paragraph{Constrained logistic estimation error.}
Because $\mathcal L_n$ is convex and differentiable and $\Theta$ is convex, the constrained minimizer satisfies the first-order variational inequality
\[
\left\langle
\nabla\mathcal L_n(\widehat\theta_n),\,
\theta-\widehat\theta_n
\right\rangle\ge0
\qquad\text{for every }\theta\in\Theta.
\]
Since $\theta^\star\in\Theta$, choosing $\theta=\theta^\star$ gives
\begin{equation}
\label{eq:theory_vi}
\left\langle
\nabla\mathcal L_n(\widehat\theta_n),\,
e_n
\right\rangle\le0,
\qquad
e_n:=\widehat\theta_n-\theta^\star.
\end{equation}
When the norm constraint is inactive, Equation~\ref{eq:theory_vi} holds with equality and reduces to the usual unconstrained score equation, so the argument below covers both cases. Set
\[
\varepsilon_m:=Y_m-\sigma(X_m^\top\theta^\star),
\qquad
S_n:=\sum_{m=1}^{n-1}\varepsilon_mX_m,
\]
and substitute $Y_m=\sigma(X_m^\top\theta^\star)+\varepsilon_m$ into the gradient
\[
\nabla\mathcal L_n(\widehat\theta_n)
=
\sum_{m=1}^{n-1}[\sigma(X_m^\top\widehat\theta_n)-Y_m]X_m
+\lambda\widehat\theta_n
\]
to obtain
\begin{equation}
\label{eq:theory_score}
\nabla\mathcal L_n(\widehat\theta_n)
=
\sum_{m=1}^{n-1}
[\sigma(X_m^\top\widehat\theta_n)-\sigma(X_m^\top\theta^\star)]X_m
+\lambda e_n
-\left(S_n-\lambda\theta^\star\right).
\end{equation}
For each $m<n$, the scalar mean-value theorem gives a point $\xi_{m,n}$ between $X_m^\top\theta^\star$ and $X_m^\top\widehat\theta_n$ such that
\[
\sigma(X_m^\top\widehat\theta_n)-\sigma(X_m^\top\theta^\star)
=
\sigma'(\xi_{m,n})X_m^\top e_n.
\]
Thus $\nabla\mathcal L_n(\widehat\theta_n)=H_ne_n-(S_n-\lambda\theta^\star)$, where
\[
H_n
:=
\lambda I+\sum_{m=1}^{n-1}\sigma'(\xi_{m,n})X_mX_m^\top,
\]
and Equation~\ref{eq:theory_vi} becomes
\[
e_n^\top H_ne_n\le e_n^\top(S_n-\lambda\theta^\star).
\]
Each $\xi_{m,n}$ equals $X_m^\top\theta_u$ for some $u\in[0,1]$, because the map $u\mapsto X_m^\top[\theta^\star+u(\widehat\theta_n-\theta^\star)]$ is affine and traverses exactly the interval between the two endpoints. Lemma~\ref{lem:theory_curvature} therefore implies
\[
H_n
\succeq
\lambda I+\underline\kappa\sum_{m=1}^{n-1}X_mX_m^\top
\succeq
\underline\kappa G_n,
\]
where the final inequality uses $0<\underline\kappa\le1/4<1$. Hence
\[
\begin{aligned}
\underline\kappa\|e_n\|_{G_n}^2
&\le e_n^\top H_ne_n\\
&\le e_n^\top(S_n-\lambda\theta^\star)\\
&\le
\|e_n\|_{G_n}
\left(
\|S_n\|_{G_n^{-1}}
+\lambda\|\theta^\star\|_{G_n^{-1}}
\right).
\end{aligned}
\]
Since $G_n\succeq\lambda I$,
\[
\lambda\|\theta^\star\|_{G_n^{-1}}
\le\sqrt\lambda\|\theta^\star\|_2
\le\sqrt\lambda S.
\]
Combining this with Equation~\ref{eq:theory_selfnorm} proves the first inequality in Equation~\ref{eq:theory_param_confidence}. The determinant bound
\[
\log\frac{\det G_n}{\lambda^d}
\le
d\log\!\left(1+\frac{(n-1)L^2}{\lambda d}\right)
\]
follows from the arithmetic--geometric mean inequality applied to the eigenvalues and yields $\beta_n\le\bar\beta_n$.

At the fitted endpoint, Lemma~\ref{lem:theory_curvature} applied with $u=1$ and the global bound $\sigma'(z)\le1/4\le1$ give
\begin{equation}
\label{eq:theory_sandwich}
\underline\kappa G_n\preceq V_n\preceq G_n.
\end{equation}
After inversion,
\[
G_n^{-1}\preceq V_n^{-1}\preceq\frac{1}{\underline\kappa}G_n^{-1}.
\]
For every candidate context $x$,
\[
\begin{aligned}
|x^\top(\widehat\theta_n-\theta^\star)|
&\le
\|\widehat\theta_n-\theta^\star\|_{G_n}\|x\|_{G_n^{-1}}\\
&\le
\beta_n(\delta)\|x\|_{V_n^{-1}}.
\end{aligned}
\]
Therefore
\[
x^\top\theta^\star
\le
x^\top\widehat\theta_n
+\beta_n(\delta)\sqrt{x^\top V_n^{-1}x},
\]
and monotonicity of the sigmoid proves Equation~\ref{eq:theory_valid_optimism}.

\paragraph{Proof of Theorem~\ref{thm:theory_realizable}.}
On the confidence event, valid optimism and UCB maximization imply
\[
\Delta U_n(J_n^\star)p_n(J_n^\star)
\le
\Delta U_n(J_n^\star)p_n^{\mathrm{UCB}}(J_n^\star)
\le
\Delta U_n(J_n)p_n^{\mathrm{UCB}}(J_n).
\]
Hence
\begin{equation}
\label{eq:theory_inst1}
r_n
\le
\Delta U_n(J_n)
[p_n^{\mathrm{UCB}}(J_n)-p_n(J_n)].
\end{equation}
Let $s_n:=\sqrt{X_n^\top V_n^{-1}X_n}$. The confidence event gives
\[
|X_n^\top(\widehat\theta_n-\theta^\star)|\le\beta_ns_n.
\]
Thus the optimistic logit exceeds the true logit by at most $2\beta_ns_n$. Since $\sup_z\sigma'(z)=1/4$, the sigmoid is globally $1/4$-Lipschitz and
\[
0\le p_n^{\mathrm{UCB}}(J_n)-p_n(J_n)
\le\frac{\beta_n}{2}s_n.
\]
Using $\Delta U_n(J_n)\le D$ in Equation~\ref{eq:theory_inst1},
\begin{equation}
\label{eq:theory_inst2}
r_n\le\frac{D\beta_n}{2}s_n.
\end{equation}
From Equation~\ref{eq:theory_sandwich},
\[
s_n^2
\le
\frac{1}{\underline\kappa}X_n^\top G_n^{-1}X_n.
\]
Set $q_n:=X_n^\top G_n^{-1}X_n$. Since $G_n\succeq\lambda I$, $0\le q_n\le L^2/\lambda$. By the matrix determinant lemma,
\[
\log\frac{\det G_{N+1}}{\det G_1}
=
\sum_{n=1}^N\log(1+q_n).
\]
For $0\le q\le L^2/\lambda$,
\[
\log(1+q)
\ge\frac{q}{1+q}
\ge\frac{q}{1+L^2/\lambda}.
\]
Therefore
\[
\sum_{n=1}^Ns_n^2
\le
\frac{d}{\underline\kappa}
\left(1+\frac{L^2}{\lambda}\right)
\log\!\left(1+\frac{NL^2}{\lambda d}\right).
\]
Cauchy--Schwarz yields
\begin{equation}
\label{eq:theory_sumwidth}
\sum_{n=1}^Ns_n
\le
\sqrt{
\frac{Nd}{\underline\kappa}
\left(1+\frac{L^2}{\lambda}\right)
\log\!\left(1+\frac{NL^2}{\lambda d}\right)
}.
\end{equation}
Summing Equation~\ref{eq:theory_inst2}, using monotonicity of $\bar\beta_n$, and applying Equation~\ref{eq:theory_sumwidth} proves Equation~\ref{eq:theory_regret_realizable}.

\paragraph{Proof of Theorem~\ref{thm:theory_misspec}.}
Under Assumption~\ref{ass:theory_misspec}, for each selected historical action define
\[
p_m:=\mathbb E[Y_m\mid\mathcal F_{m-1}],
\qquad
\xi_m:=Y_m-p_m.
\]
Then $\mathbb E[\xi_m\mid\mathcal F_{m-1}]=0$, and $\xi_m$ is conditionally $1/2$-sub-Gaussian. Also
\[
Y_m=\sigma(X_m^\top\theta^\star)+b_m+\xi_m.
\]
Assumption~\ref{ass:theory_misspec} gives $\theta^\star\in\Theta$, so the variational inequality of Equation~\ref{eq:theory_vi} applies verbatim. Repeating the mean-value expansion with $Y_m$ decomposed as above gives
\[
\nabla\mathcal L_n(\widehat\theta_n)
=
H_ne_n
-\left(
\sum_{m=1}^{n-1}\xi_mX_m
+
\sum_{m=1}^{n-1}b_mX_m
-
\lambda\theta^\star
\right),
\]
and therefore
\[
e_n^\top H_ne_n
\le
e_n^\top\!\left(
\sum_{m=1}^{n-1}\xi_mX_m
+
\sum_{m=1}^{n-1}b_mX_m
-
\lambda\theta^\star
\right).
\]
The martingale term obeys the same specialized self-normalized inequality above. For the bias term, let $\mathbf X$ be the matrix whose $m$th row is $X_m^\top$ and let $b=(b_1,\ldots,b_{n-1})^\top$. Then
\[
\sum_{m=1}^{n-1}b_mX_m=\mathbf X^\top b,
\qquad
G_n=\lambda I+\mathbf X^\top\mathbf X.
\]
The eigenvalues of $\mathbf XG_n^{-1}\mathbf X^\top$ are $s_k^2/(\lambda+s_k^2)\le1$, so
\[
\begin{aligned}
\|\mathbf X^\top b\|_{G_n^{-1}}^2
&=
b^\top\mathbf XG_n^{-1}\mathbf X^\top b\\
&\le\|b\|_2^2
\le\sum_{m=1}^{n-1}\eta_m^2.
\end{aligned}
\]
The same curvature argument based on Lemma~\ref{lem:theory_curvature} therefore gives
\[
\|\widehat\theta_n-\theta^\star\|_{G_n}
\le\beta_n^{\mathrm{mis}}(\delta).
\]
Now define the logistic reference $\bar p_n(j):=\sigma(x_{j,n}^\top\theta^\star)$ and its UCB
\[
u_n(j)
:=
\sigma\!\left(
x_{j,n}^\top\widehat\theta_n
+\beta_n^{\mathrm{mis}}s_n(x_{j,n})
\right).
\]
The confidence argument gives $\bar p_n(j)\le u_n(j)$, whereas misspecification gives $p_n(j)\le\bar p_n(j)+\eta_n$. For the true oracle action $J_n^\star$,
\[
\begin{aligned}
\Delta U_n(J_n^\star)p_n(J_n^\star)
&\le
\Delta U_n(J_n^\star)u_n(J_n^\star)+D\eta_n\\
&\le
\Delta U_n(J_n)u_n(J_n)+D\eta_n.
\end{aligned}
\]
Subtracting the true expected gain of the selected action gives
\[
\begin{aligned}
r_n
&\le
\Delta U_n(J_n)[u_n(J_n)-p_n(J_n)]+D\eta_n\\
&\le
\Delta U_n(J_n)[u_n(J_n)-\bar p_n(J_n)]+2D\eta_n.
\end{aligned}
\]
The same $1/4$-Lipschitz argument yields
\[
u_n(J_n)-\bar p_n(J_n)
\le\frac{\beta_n^{\mathrm{mis}}}{2}s_n,
\]
so
\[
r_n
\le
\frac{D\beta_n^{\mathrm{mis}}}{2}s_n
+2D\eta_n.
\]
Summing and applying Equation~\ref{eq:theory_sumwidth} proves Equation~\ref{eq:theory_regret_mis}.

For the stated sufficient condition, if $0\le\eta_n\le1$ and $\sum_{n=1}^N\eta_n=o(N/\log N)$, then
\[
\sum_{n=1}^N\eta_n^2
\le
\sum_{n=1}^N\eta_n
=
o(N/\log N).
\]
Hence $\sqrt{\sum_{n=1}^N\eta_n^2}=o(\sqrt{N/\log N})$. Its contribution through $\bar\beta_N^{\mathrm{mis}}$ times the width sum is $o(N)$; the explicit $\sum_n\eta_n$ term is also $o(N)$, while the ordinary statistical term is sublinear for fixed dimension and constants. Therefore $\mathcal R_N/N\to0$.

\paragraph{Proof of Proposition~\ref{prop:theory_bilateral}.}
\begin{proof}
Because $j$ belongs to the proposer's eligible set,
\[
\Delta U_{ij,t}=U_i(j)-b_{i,t}>\epsilon_U\ge0,
\]
so $U_i(j)>b_{i,t}$. If $i$ is single, $b_{i,t}=r_i=v_i(\mu_t)$. If $i$ is currently matched,
\[
b_{i,t}=\max\{r_i,U_i(\mu_t(i))\}\ge U_i(\mu_t(i))=v_i(\mu_t).
\]
Thus $U_i(j)>v_i(\mu_t)$. Acceptance by $j$ means $U_j(i)>b^R_{j,t}$. If $j$ is single, $b^R_{j,t}=r_j=v_j(\mu_t)$; if $j$ is matched,
\[
b^R_{j,t}=\max\{r_j,U_j(\mu_t(j))\}\ge U_j(\mu_t(j))=v_j(\mu_t).
\]
Hence $U_j(i)>v_j(\mu_t)$ as well.
\end{proof}

\paragraph{Proof of Proposition~\ref{prop:theory_nonmono}.}
\begin{proof}
Consider $M=\{m_1,m_2\}$ and $W=\{w_1,w_2\}$, with reservation utilities below $0.60$ and all current relationships temporary. Let
\[
\mu_t=\{(m_1,w_1),(m_2,w_2)\},
\]
and choose directional utilities
\[
U_{m_1}(w_1)=0.60,
\quad U_{w_1}(m_1)=0.99,
\quad U_{m_2}(w_2)=0.99,
\quad U_{w_2}(m_2)=0.60,
\]
with cross-pair values
\[
U_{m_1}(w_2)=0.65,
\qquad
U_{w_2}(m_1)=0.65.
\]
Take $\epsilon_U<0.05$. Then $m_1$ strictly prefers $w_2$ to its current benchmark and $w_2$ accepts $m_1$, so the proposal is a strict bilateral improvement for the consenting agents. Before the proposal,
\[
W_{\mathrm{tot}}(\mu_t)
=
\frac{0.60+0.99}{2}
+
\frac{0.99+0.60}{2}
=1.59,
\qquad
W_{\mathrm{avg}}(\mu_t)=0.795.
\]
After acceptance, the two old temporary relationships are dissolved, $w_1$ and $m_2$ are displaced, and immediately after rematching
\[
\mu_{t+1}=\{(m_1,w_2)\}.
\]
Hence
\[
W_{\mathrm{tot}}(\mu_{t+1})
=
W_{\mathrm{avg}}(\mu_{t+1})
=
\frac{0.65+0.65}{2}
=0.65.
\]
Both total matched-pair welfare and average mutual welfare therefore decrease. The decline is caused by the losses of the displaced partners, which are not part of the bilateral acceptance test.
\end{proof}

\paragraph{Proof of Corollary~\ref{cor:theory_local_global}.}
\begin{proof}
The local regret theorems control only the loss in one-step expected proposer improvement relative to a myopic acceptance-probability oracle. Proposition~\ref{prop:theory_nonmono} constructs a feasible accepted rematching in which both consenting agents improve while matched-pair welfare decreases. Therefore local decision quality, even if asymptotically no-regret, does not by itself imply monotonicity or terminal optimality of the matched-pair welfare metric. Stability is likewise not implied because the local regret benchmark contains no stability constraint.
\end{proof}

\section{Full Decentralized Matching Procedure}
\label{app:full_matching_algorithm}

Algorithm~\ref{alg:decentralized_matching} expands the period-level procedure summarized in Section~\ref{sec:matching_protocol}.
It makes explicit the order of local exposure, valuation, proposal, feedback, model updating, relationship revision, and stochastic commitment.

\begin{algorithm}[htbp]
\caption{Decentralized LLM-Agent Matching with Local Logistic-UCB}
\label{alg:decentralized_matching}
\begin{algorithmic}[1]

\Require
Agent sets $\mathcal{M},\mathcal{W}$; horizon $T$; proposal budget $L$; reservation utilities $\{r_i\}$; regularization $\{\lambda_i\}$; exploration coefficients $\{\beta_{i,t}\}$; utility threshold $\epsilon_U$; commitment probabilities $\{\rho_t\}$.

\State Initialize all agents as single; set local histories empty and initialize local Logistic-UCB models.

\For{$t=1,\ldots,T$}
    \State $\mathcal{A}^{\mathrm{act}}_t \gets \{i\in\mathcal{A}: i \text{ is not permanently matched}\}$
    \If{$\mathcal{A}^{\mathrm{act}}_t=\varnothing$}
        \State \textbf{break}
    \EndIf

    \For{$a=1,\ldots,|\mathcal{A}^{\mathrm{act}}_t|$}
        \State Sample $i\sim\mathrm{Uniform}(\mathcal{A}^{\mathrm{act}}_t)$ with replacement
        \State Sample local exposure set $\mathcal{C}_t(i)\subseteq\mathcal{P}_t(i)$
        \State Obtain $U_i(j)$ for all $j\in\mathcal{C}_t(i)$ and construct $\mathcal{E}_t(i)$ using Equation~\ref{eq:eligible_candidate_set}
        \State $\mathcal{R}_t \gets \mathcal{E}_t(i)$

        \For{$q=1,\ldots,L$}
            \If{$\mathcal{R}_t=\varnothing$}
                \State \textbf{break}
            \EndIf

            \State For each $j\in\mathcal{R}_t$, construct $x^{(i)}_{j,t}$ and compute $I_{ij,t}$ using Equations~\ref{eq:logistic_ucb} and~\ref{eq:proposal_index}
            \State $j^*\gets\arg\max_{j\in\mathcal{R}_t} I_{ij,t}$
            \State Obtain $U_{j^*}(i)$ and compute $y^*\gets Y_{i\rightarrow j^*,t}$ using Equation~\ref{eq:binary_acceptance}
            \State $\mathcal{H}_{i,t}\gets\mathcal{H}_{i,t}\cup\{(x^{(i)}_{j^*,t},y^*)\}$
            \State Refit $\widehat\theta_{i,t}$ by Equation~\ref{eq:appendix_logistic_loss} and update $V_{i,t}$ by Equation~\ref{eq:appendix_fisher_matrix}

            \If{$y^*=1$}
                \State Dissolve temporary matches involving $i$ or $j^*$ and return displaced partners to active search
                \State Form $(i,j^*)$ as a new temporary pair
                \State \textbf{break}
            \Else
                \State $\mathcal{R}_t\gets\mathcal{R}_t\setminus\{j^*\}$
            \EndIf
        \EndFor
    \EndFor

    \State Let $\mathcal{T}_t$ denote temporary pairs after all activations
    \ForAll{$(i,j)\in\mathcal{T}_t$ \textbf{ in parallel}}
        \State Sample $Z_{ij,t}\sim\mathrm{Bernoulli}(\rho_t)$
        \If{$Z_{ij,t}=1$}
            \State Mark $(i,j)$ as permanent
        \EndIf
    \EndFor
\EndFor

\State \Return final matching $\mu$

\end{algorithmic}
\end{algorithm}

The procedure preserves three distinct forms of locality.
Candidate discovery is local through $\mathcal{C}_t(i)$; LLM valuation is triggered only when a candidate becomes locally relevant; and reciprocal-acceptance learning is agent-specific because only proposals made by $i$ enter $\mathcal{H}_{i,t}$.
No global acceptance model, shared interaction history, or market-wide preference matrix is available to the proposed agents.

\section{Limitations and Future Directions}
\label{app:limitations_future}

\subsection{Limitations}

The empirical grounding of this study is intentionally specific to the Chinese marriage-market setting represented by the 2015 population sample and the choice experiment of \citet{zhou2023gender}. Accordingly, the validation establishes alignment with this empirical reference rather than universal behavioral validity across cultures, cohorts, or institutional settings. Extending the framework to other populations would require corresponding domain-specific demographic construction and behavioral validation.

The simulated market also remains deliberately stylized. The main experiments use a $50\times50$ population with random local exposure, a common reservation-utility specification, and probabilistic commitment. Richer interaction structures---including social networks, endogenous meeting processes, heterogeneous or adaptive reservation values, noisy or private attributes, and evolving outside options---may alter the resulting matching dynamics. These components can be incorporated within the same decentralized framework, but are outside the scope of the present experiments.

The theoretical results in Appendix~\ref{app:theoretical_analysis} characterize local proposal-level learning and bilateral rematching; they do not guarantee terminal welfare optimality or convergence to a stable matching. Establishing endogenous bounds on dynamic misspecification and conditions linking local learning to market-level outcomes remains an open direction. Moreover, learned acceptance coefficients represent predictive associations rather than structural or causal preferences, because proposal observations are endogenously selected by the matching policy.

Finally, LLM-based behavioral modeling remains sensitive to persona construction, prompt design, model choice, and demographic representation \citep{anthis2025position,zhou2026the,ludwig2024large}. Our cross-backbone and empirical-reference validation reduces dependence on any single model realization, but does not eliminate these broader sources of behavioral-model uncertainty.

\subsection{Future Directions}

The framework naturally extends to other decentralized economic settings in which agents evaluate counterparties while learning whether interaction will be reciprocated, including labor-market search, hiring, team formation, housing, platform matching, and repeated buyer--seller or lender--borrower interactions.
The same separation between domain-grounded behavioral valuation and local reciprocal learning can be retained while replacing personas, actions, institutions, and validation targets with domain-specific counterparts.

A second direction is to enrich the interaction structure through social networks, geography, recommender systems, endogenous communication, longer memories, adaptive reservation values, or information sharing.
On the learning side, nonstationary bandits, hierarchical priors, and theoretical conditions connecting local learning to welfare or stability are natural extensions.
Beyond economics, the architecture can support persistent multi-agent societies in games, where NPCs form and revise friendships, rivalries, coalitions, or trading relationships through local interaction rather than a globally scripted social graph.

\end{document}